\documentclass{article}

\usepackage[T1]{fontenc}

\usepackage{graphicx} 
\usepackage{geometry}
\usepackage{amsmath, amsfonts, amssymb, amsthm}
\usepackage[subpreambles=false]{standalone}

\usepackage{cite}
\usepackage{etoolbox}
\usepackage{enumerate}
\usepackage{xcolor}
\usepackage{breqn}
\usepackage{xcolor}
\usepackage{import}
\usepackage{standalone}
\usepackage{subcaption}
\usepackage[font=small,labelfont=bf]{caption}
\usepackage[toc,page]{appendix}
\usepackage{placeins}
\usepackage{ifthen}
\usepackage{anyfontsize}
\usepackage{tikz}
\usepackage[inline]{enumitem}
\usepackage[linesnumbered,ruled,lined,boxed, noend]{algorithm2e}
\usepackage{hyperref}
\usepackage{mathtools}
\usepackage[noabbrev,capitalize]{cleveref}
\usepackage{tabularx}

\SetAlgoSkip{smallskip}

\usetikzlibrary{calc}

\definecolor{darkgreen}{rgb}{0.0, 0.6, 0.0}
\definecolor{darkred}{rgb}{0.85, 0.0, 0.0}
\colorlet{ZeroColor}{red!70!black}
\colorlet{OneColor}{green!50!black}
\colorlet{TwoColor}{blue!70!black}
\definecolor{ThreeColor}{RGB}{255, 165, 0}

\newcommand{\defeq}{\vcentcolon=}

\newtheorem{definition}{Definition}[section]

\newtheorem{lemma}[definition]{Lemma}
\newtheorem{theorem}[definition]{Theorem}
\newtheorem{corollary}[definition]{Corollary}

\newcommand\Set[2]{\{\,#1\mid#2\,\}}

\newcommand\Robot{\mathcal{R}}
\newcommand\Sense[2]{\mathbb{S}(#1, #2)}
\newcommand\Projection[3]{\mathbb{P}_{#1}(#2, #3)}
\newcommand\Dcube{\mathcal{Q}}
\newcommand\SenseGraph[2]{G(#1, #2)}

\newcommand\Position{p}

\makeatletter
\newcommand{\nosemic}{\renewcommand{\@endalgocfline}{\relax}}
\newcommand{\dosemic}{\renewcommand{\@endalgocfline}{\algocf@endline}}
\makeatother

\newcommand*{\argmin}{\operatornamewithlimits{argmin}\limits}

\pgfmathsetmacro{\upperrowmult}{1/3}
\pgfmathsetmacro{\XSize}{20}
\pgfmathsetmacro{\YSize}{20}
\pgfmathsetmacro{\ZSize}{20}
\pgfmathsetmacro{\isometricangle}{40}
\pgfmathsetmacro{\tanofangle}{tan(\isometricangle)}
\pgfmathsetmacro{\opacitylevel}{0.3}
\pgfmathsetmacro{\layerstep}{2}
\pgfmathsetmacro{\interlayerstep}{6}

\newcommand{\DrawGridFace}[2]{
    \path[box_frame] (0, 0) rectangle (#1, #2);
    \path[dashed_grid] (0, 0) grid (#1, #2);
}

\newcommand{\DrawGridBox}[5]{
    \begin{scope}[shift={(#1, #2)}]
        \DrawGridFace{#3}{#5}
        \begin{scope}[shift={(0, #5)}]
            \begin{scope}[cm={1,0,tan(90-\isometricangle),1,(0,0)}]
                \begin{scope}[xscale=1, yscale=\upperrowmult]
                    \DrawGridFace{#3}{#4}
                \end{scope}
            \end{scope}
        \end{scope}
        \begin{scope}[shift={(#3, 0)}]
            \begin{scope}[cm={1,tan(\isometricangle),0,1,(0,0)}]
                \pgfmathsetmacro{\xstep}{\upperrowmult / \tanofangle}
                \begin{scope}[xscale=\xstep]
                    \DrawGridFace{#4}{#5}
                \end{scope}
            \end{scope}
        \end{scope}
    \end{scope}
}

\newcommand{\DrawCubeCorner}[2]{
    \begin{scope}[shift={(#1, #2)}]
        \foreach \yrow in {1, ..., \VisibilityRange} {
            \pgfmathsetmacro{\yxcoord}{\upperrowmult * tan(90 - \isometricangle) * (\yrow - 1)}
            \pgfmathsetmacro{\yzcoord}{\upperrowmult * (\yrow - 1)}
            \foreach \zrow in {1, ..., \VisibilityRange} {
                \foreach \xrow in {1, ..., \VisibilityRange} {
                    \pgfmathsetmacro{\xcoord}{\layerstep * (\xrow - 1) - \layerstep * \yxcoord}
                    \pgfmathsetmacro{\zcoord}{\layerstep * (\zrow - 1) - \layerstep * \yzcoord}
                    \DrawGridBox{\xcoord}{\zcoord}{1}{1}{1}        
                }
            }
        }
    \end{scope}
}

\newcommand{\DrawCubeEdgesAlongX}[2]{
    \begin{scope}[shift={(#1, #2)}]
        \foreach \yrow in {1, ..., \VisibilityRange} {
            \pgfmathsetmacro{\yxcoord}{\upperrowmult * tan(90 - \isometricangle) * (\yrow - 1)}
            \pgfmathsetmacro{\yzcoord}{\upperrowmult * (\yrow - 1)}
            \foreach \zrow in {1, ..., \VisibilityRange} {
                \pgfmathsetmacro{\xcoord}{-\layerstep * \yxcoord}
                \pgfmathsetmacro{\zcoord}{\layerstep * (\zrow - 1) - \layerstep * \yzcoord}
                \DrawGridBox{\xcoord}{\zcoord}{\XSize}{1}{1}        
            }
        }
    \end{scope}
}

\newcommand{\DrawCubeEdgesAlongY}[2]{
    \begin{scope}[shift={(#1, #2)}]
        \pgfmathsetmacro{\yxoffset}{\upperrowmult * tan(90 - \isometricangle) * (\YSize - 1)}
        \pgfmathsetmacro{\yzoffset}{\upperrowmult * (\YSize - 1)}
        \foreach \zrow in {1, ..., \VisibilityRange} {
            \foreach \xrow in {1, ..., \VisibilityRange} {
                \pgfmathsetmacro{\xcoord}{\layerstep * (\xrow - 1) - \yxoffset}
                \pgfmathsetmacro{\zcoord}{\layerstep * (\zrow - 1) - \yzoffset}
                \DrawGridBox{\xcoord}{\zcoord}{1}{\YSize}{1}        
            }
        }
    \end{scope}
}

\newcommand{\DrawCubeEdgesAlongZ}[2]{
    \begin{scope}[shift={(#1, #2)}]
        \foreach \yrow in {1, ..., \VisibilityRange} {
            \pgfmathsetmacro{\yxcoord}{\upperrowmult * tan(90 - \isometricangle) * (\yrow - 1)}
            \pgfmathsetmacro{\yzcoord}{\upperrowmult * (\yrow - 1)}
            \foreach \xrow in {1, ..., \VisibilityRange} {
                \pgfmathsetmacro{\xcoord}{\layerstep * (\xrow - 1) - \layerstep * \yxcoord}
                \pgfmathsetmacro{\zcoord}{-\layerstep * \yzcoord}
                \DrawGridBox{\xcoord}{\zcoord}{1}{1}{\ZSize}        
            }
        }
    \end{scope}
}

\newcommand{\DrawCubeFacesAlongX}[2]{
    \begin{scope}[shift={(#1, #2)}]
        \pgfmathsetmacro{\yxoffset}{\upperrowmult * tan(90 - \isometricangle) * (\YSize - 1)}
        \pgfmathsetmacro{\yzoffset}{\upperrowmult * (\YSize - 1)}
        \foreach \xrow in {1, ..., \VisibilityRange} {
            \pgfmathsetmacro{\xcoord}{\layerstep * (\xrow - 1) - \yxoffset}
            \pgfmathsetmacro{\zcoord}{-\yzoffset}
            \DrawGridBox{\xcoord}{\zcoord}{1}{\YSize}{\ZSize}        
        }
    \end{scope}
}

\newcommand{\DrawCubeFacesAlongY}[2]{
    \begin{scope}[shift={(#1, #2)}]
        \foreach \yrow in {1, ..., \VisibilityRange} {
            \pgfmathsetmacro{\yxcoord}{\upperrowmult * tan(90 - \isometricangle) * (\yrow - 1)}
            \pgfmathsetmacro{\yzcoord}{\upperrowmult * (\yrow - 1)}
            \pgfmathsetmacro{\xcoord}{-\layerstep * \yxcoord}
            \pgfmathsetmacro{\zcoord}{-\layerstep * \yzcoord}
            \DrawGridBox{\xcoord}{\zcoord}{\XSize}{1}{\ZSize}        
        }
    \end{scope}
}

\newcommand{\DrawCubeFacesAlongZ}[2]{
    \begin{scope}[shift={(#1, #2)}]
        \pgfmathsetmacro{\yxoffset}{\upperrowmult * tan(90 - \isometricangle) * (\YSize - 1)}
        \pgfmathsetmacro{\yzoffset}{\upperrowmult * (\YSize - 1)}
        \foreach \zrow in {1, ..., \VisibilityRange} {
            \pgfmathsetmacro{\xcoord}{-\yxoffset}
            \pgfmathsetmacro{\zcoord}{\layerstep * (\zrow - 1) - \yzoffset}
            \DrawGridBox{\xcoord}{\zcoord}{\XSize}{\YSize}{1}        
        }
    \end{scope}
}

\newcommand{\DrawCubeInnerCore}[2]{
    \begin{scope}[shift={(#1, #2)}]
        \pgfmathsetmacro{\yxoffset}{\upperrowmult * tan(90 - \isometricangle) * (\YSize - 1)}
        \pgfmathsetmacro{\yzoffset}{\upperrowmult * (\YSize - 1)}
        \pgfmathsetmacro{\xcoord}{-\yxoffset}
        \pgfmathsetmacro{\zcoord}{-\yzoffset}
        \DrawGridBox{\xcoord}{\zcoord}{\XSize}{\YSize}{\ZSize}        
    \end{scope}
}

\title{Patrolling Grids with a Bit of Memory}

\usepackage{authblk} 

\begin{document}

\author[1,2]{Michael Amir}
\author[1]{Dmitry Rabinovich}
\author[1]{Alfred M. Bruckstein}

\affil[1]{University of Cambridge, Cambridge, United Kingdom}
\affil[2]{Technion - Israel Institute of Technology,  Haifa, Israel}

\affil[ ]{Emails: \texttt{ma2151@cam.ac.uk} (M. Amir), \texttt{dmitry.ra@technion.ac.il} (D. Rabinovich), \texttt{freddy@cs.technion.ac.il} (A. M. Bruckstein)}

\maketitle

\begin{abstract}
This work addresses the challenge of patrolling regular grid graphs of any dimension using a single mobile agent with minimal memory and limited sensing range. We show that it is impossible to patrol some grid graphs with $0$ bits of memory, regardless of sensing range, and give an exact characterization of those grid graphs that can be patrolled with $0$ bits of memory and sensing range $V$. On the other hand, we show that an algorithm exists using $1$ bit of memory and $V=1$ that patrols any $d$-dimensional grid graph. This result is surprising given that the agent must be able to move in $2d$ distinct directions to patrol, while $1$ bit of memory  allows specifying only two directions per sensory input. Our $1$-bit patrolling algorithm handles this by carefully exploiting a small state-space to access all the needed directions while avoiding getting stuck. Overall, our results give concrete evidence that extremely little memory is needed for patrolling highly regular environments like grid graphs compared to arbitrary graphs. The techniques we use, such as partitioning the environment into sensing regions and exploiting distinct coordinates resulting from higher-dimensionality, may be applicable to analyzing the space complexity of patrolling in other types of regular environments as well.
\end{abstract}

\section{Introduction}

Patrolling is a key problem in robotics and  operations research wherein a mobile agent or team of mobile agents are tasked with repeatedly visiting  every vertex of a graph environment by traversing edges. This task has well-known applications to navigation, warehouse management, web crawling, and swarm intelligence \cite{gkasieniec2008memory}. Patrolling has been studied under diverse sets of assumptions regarding e.g. the number of agents, the capabilities of each agent, and the underlying graph environment \cite{izumi2022deciding}. A central  question related to patrolling is the \textit{space complexity} of patrolling an environment: the amount of memory required by the agent(s) to patrol the environment \cite{izumi2022deciding,fraigniaud2005graph,dobrev2005finding,beame1998time,diks2004tree,gasieniec2007tree,cohen2017exploring}. Rectangular grid graphs are fundamental in robotics, and are the subject of many works in patrolling \cite{bampas2010almost,cohen2017exploring,dobrev2019exploration}. Despite this fact, and despite the space complexity of patrolling being widely studied, the space complexity of patrolling grid graphs by a single agent has not yet been established.  Hence, the goal of this work is to establish the space complexity of patrolling $d$-dimensional grid graphs with a single mobile agent that has limited visibility.

More concretely, suppose a mobile agent with \textit{fixed orientation}, $\Robot$, is placed somewhere inside a grid graph. We assume $\Robot$ has $b$ bits of state memory that persists between steps and the ability to see locations at Manhattan distance $V$ or less from itself (see \cref{fig:agent.visibility.examples}), and must act based only on this information, i.e., it is a finite automaton enhanced with local geometric information. For what values of $b$ and $V$ does there exist an algorithm that enables $\Robot$ to patrol the grid?

\begin{figure}[ht]
    \centering
    \begin{subfigure}{0.26\textwidth}
        \centering
        \resizebox{\textwidth}{!}{
            \includegraphics[page=1]{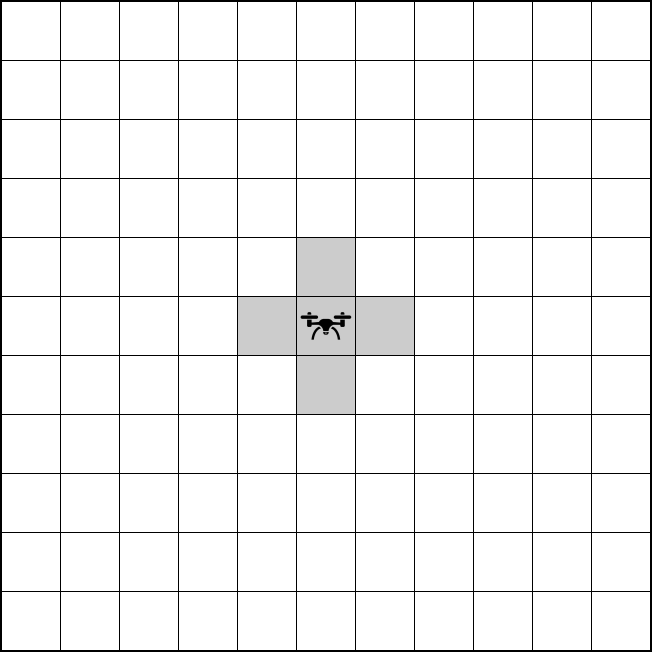}
        }
        \caption{}
        \label{subfig:robot.v=1}
    \end{subfigure}
    \begin{subfigure}{0.26\textwidth}
        \centering
        \resizebox{\textwidth}{!}{
            \includegraphics[page=1]{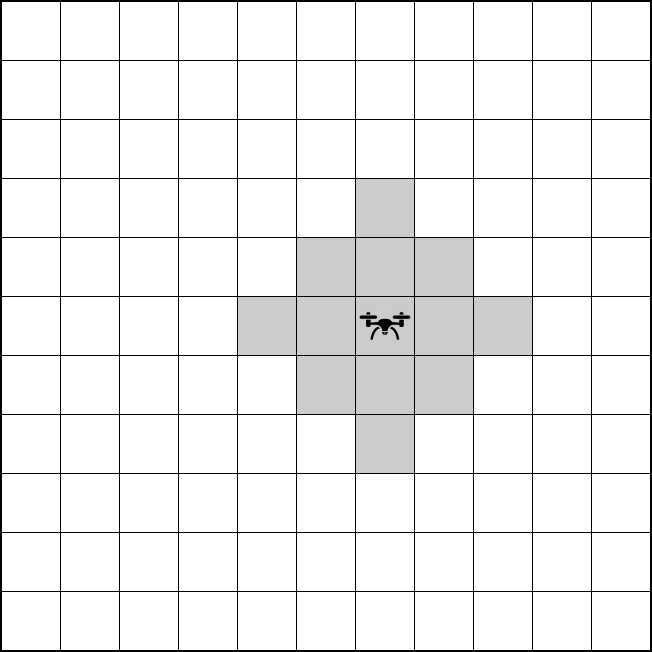}
        }
        \caption{}
        \label{subfig:robot.v=2}
    \end{subfigure}
    \caption{Gray-colored vertices illustrate what our mobile agent can see with sensing range $V=1$ and $V=2$, respectively, in a $2$-dimensional grid graph.}
    \label{fig:agent.visibility.examples}
\end{figure}

One might expect that existing space complexity results for broader classes of environments give good complexity bounds for grid graphs as a special case  - but this does not appear to be true. The strongest result we found in the literature implies a $d$-dimensional grid graph $G$ of diameter $\mathrm{diam}(\mathcal{G})$ can be patrolled with $\mathcal{O}(\mathrm{diam}(\mathcal{G}) \cdot log(d))$ bits \cite{fraigniaud2005graph,fraigniaud2005space}. We recount this result in  \cref{appendix:patrollingingeneral} for the sake of comparison. Generally, existing works set the expectation that the complexity of patrolling a given graph environment grows with diameter and maximum vertex degree (see "Related Work"). On the contrary, we show that, irrespective of diameter and maximum vertex degree (i.e., dimension), there is an algorithm using $1$ bit of memory that patrols all grid graphs. 


\cref{theorem:grid graphsthatcanbepatrolledwithVvisibility0bits} and \cref{theorem:canpatrolwith1bit} together characterize exactly which grid graphs can be patrolled with sensing range $V$ and $b$ bits of memory, settling our main question.  \cref{theorem:grid graphsthatcanbepatrolledwithVvisibility0bits} gives an exact characterization of $d$-dimensional grid graphs that can be patrolled with sensing range $V$ and $0$ bits of memory. These turn out to be grid graphs that are bounded in all dimensions but at most one, and that have an even number of ``sensing regions'' (regions of vertices that are indistinguishable to $\Robot$). Our proof idea involves partitioning grids into sensing regions and defining a graph based on these regions' adjacency relations, which we show is itself isomorphic to a grid graph. We use this graph to augment the well-known fact that a grid with an odd number of vertices contains no Hamiltonian cycle. 

\SetKwFunction{MakeMoveAll}{MakeMove}

\cref{theorem:canpatrolwith1bit} establishes the correctness of an algorithm that patrols any $d$-dimensional grid graph with sensing range $V = 1$ and $1$ bit of memory. This is despite the fact that, to patrol a $d$-dimensional grid graph, $\Robot$ is required to move in $2d$ directions  (it must be able to increment and decrement its coordinate  along any of the $d$ axes of the grid), which might suggest that the minimum memory required should grow with $d$, and the fact that in a large grid graph, most locations are geometrically indistinguishable to $\Robot$, hence it is unable to localize itself most of the time. Consequently, we consider this result quite unexpected.

The patrolling algorithm we describe is called \MakeMoveAll. We must be more careful when patrolling low-dimensional grids than high-dimensional grids, because low-dimensional grids have fewer unique grid boundaries, and so there is less geometric information for $\Robot$ to exploit. For this reason, \MakeMoveAll handles $1$-, $2$-, and $3$-dimensional grids as special cases, but extends generically to higher dimensions by exploiting the increased number of grid boundaries. The core of the algorithm is in patrolling 2D sub-grids (subspaces of the $d$-dimensional grid of interest), and then incrementing or decrementing the robot's coordinate in some higher dimension. However, due to memory constraints, the sequence of higher-dimensional coordinate changes is not trivial, and the robot may, e.g., patrol part of a lower-dimensional subgrid only to depart and return to it later. This is to maintain a small economy of \textit{transient states} that are used by the robot to access higher dimensions.

Several of our techniques can potentially be applied to non-grid graph settings: first, the idea of partitioning the environment into sensing regions readily generalizes to non-grid environments (as noted in \cref{appendix:patrollingingeneral}) and may yield further space complexity results in such environments. Next, we believe our $1$-bit algorithm can be applied to reducing the space complexity of traversing Cartesian products of non-grid subspaces of $Z^d$, since such products are in many ways analogous to $d$-dimensional grid graphs. We can use the $1$-bit algorithm to dictate the overall ``order'' in which one should traverse the product's coordinates, without keeping track of our global position within the product space. Finally, the observation that patrolling, counter-intuitively, does not require more memory in higher-dimensional grids (because their richer, more distinguishable geometry compensates for their dimensionality) seems exploitable, and may lead to space-efficient patrolling algorithms in more general settings.

\subsection*{Related Work}

Graph exploration by mobile agents is a key topic in algorithm design, robotics, and web crawling \cite{gkasieniec2008memory}. Exploration of an environment with mobile agents presents a considerably different challenge from classical graph search settings. In classical settings, it is assumed that the entire graph can be accessed from memory. Mobile agents, on the other hand, must make decisions using limited memory and only local knowledge of the graph environment. Throughout the years,  considerable work has been devoted to identifying the minimal capabilities required to enable exploration under a diverse set of assumptions, e.g., single agent \cite{dobrev2005finding,cohen2008label} versus group of agents \cite{bampas2010almost,dobrev2019exploration,cohen2017exploring,icalpchalopin2010rendezvous}; exploring graphs with termination conditions \cite{fraigniaud2005space,gasieniec2007tree} versus without   \cite{gkasieniec2008fast,shimoyama2022one}; competitive \cite{bui2023optimal,shats2023competitive} or cooperative \cite{icalphaeupler2019optimal}. In this work we focus on non-terminating grid exploration by a single finite automaton, i.e., \textit{patrolling},  where the aim of the exploration is to ensure every graph vertex is periodically visited by the agent in some deterministic manner.


Exploration of known and unknown graphs using finite automata with limited visibility has been studied under various assumptions. In \cite{fraigniaud2005graph,fraigniaud2005space}, it is shown that a graph $G$ with maximum degree $d$ and of diameter $\mathrm{diam}(\mathcal{G})$ can be patrolled with $\mathcal{O}(\mathrm{diam}(\mathcal{G}) \cdot log(d))$. We recount this result in  \cref{appendix:patrollingingeneral} due to its relevance.  In a seminal work by Budach et al. \cite{budach1978automata} it is shown that no finite automaton exists that can find its way out of any 2D maze. More recently, Kilibarda has shown that any automaton walking rectangular labyrinth can be reduced to an automaton following either left-, or right-hand-on-the wall approach  \cite{kilibarda2017reduction}. It is known that finite automata can patrol graphs if the graph is pre-processed so as to label the edges \cite{ilcinkas2008setting,gkasieniec2008fast,cohen2008label,dobrev2005finding,czyzowicz2012more,alon1990universal}. Exploring arbitrary graphs without such pre-processing is known to require considerably more memory, and, in most problem settings, cannot be done by finite automata  \cite{diks2004tree,gasieniec2007tree,fraigniaud2005graph,fraigniaud2005space}. A particularly popular setting enhances the graph exploration capabilities of finite automata by letting them place ``pebbles''  \cite{bender1998power,disser2016undirected,beame1998time}. 

The problem of low-complexity patrolling of regular grid graphs that we study in this paper has surprising applications in autonomous traffic management \cite{rabinovich2022emerging}, where it can be used for high-level coordination of fleets of cars in a low-complexity, error-resilient fashion, and has been discussed in the context of optimal paths for lawn mowing by a single robot  \cite{arkin2000approximation}.


Our results show that, for the purposes of grid exploration, an agent with $1$ bit of memory is significantly more powerful than an agent with $0$ bits of memory (also called an ``oblivious'' agent). Some other recent works have shown similar gaps between $0$- and $1$-bit agents in graph search tasks such as deciding graph properties \cite{izumi2022deciding} and exploring cactus graphs \cite{shimoyama2022one}. The authors find fascinating this enormous leap in capabilities between memoryless agents and agents that have any internal memory at all, even a single bit.

\section{Problem Statement}

A $d$-dimensional grid graph $\Dcube$ is a graph whose vertex set is $[n_1] \times [n_2] \times \ldots [n_d]$, $\forall i, n_i > 1$, where $[n]$ is the set of integers $k$ such that $1 \leq k \leq n$, and where there is an edge between any two vertices at Manhattan distance $1$ from each other. A mobile agent,  $\Robot$, is initialized at some arbitrary vertex $v_0 \in \Dcube$ and is tasked with visiting every vertex in $\Dcube$ by traversing edges. Hence, $\Robot$ traverses $\Dcube$ by incrementing or decrementing one of $d$ different coordinates in unit steps.

Our model of movement and sensing is standard in works on single- and multi-robot navigation \cite{hsiang_algorithms_2004,agmon2009multi,amir_minimizing_2019,amir_rappel2023stigmergy}. We assume $\Robot$ has fixed orientation (i.e., a ``compass''  - it distinguishes  between the $d$ axes of $\Dcube$ and chooses which one it wants to move along) and that, when $\Robot$ is located at $\Position$, it \textit{senses} vertices of $\Dcube$ at Manhattan distance $V$ or less from $\Position$. $\Robot$ knows the \textit{position relative to $p$} of every vertex of $\Dcube$ it senses, given by the set:

\begin{equation*}
\Sense{V}{\Position} \defeq \Set{\Position' - \Position}{\Position' \in \Dcube, \lVert \Position' - \Position \rVert_1 \leq V}
\end{equation*}

where $\Position' - \Position$ is the position of $\Position'$ relative to $\Position$. $V$ is called $\Robot$'s \textit{sensing} or \textit{visibility range}. We further assume $\Robot$ has access to $b$ bits of persistent state memory which it can access and modify  - these bits represent a persistent state that $\Robot$ maintains between steps.  At every step, $\Robot$ must decide its next action \textit{deterministically} based only on $\Sense{V}{\Position}$ and its $b$ bits of memory. Hence, it can be considered a finite automaton enhanced with local geometric information. More formally, $\Robot$'s algorithm is restricted to be  a function $\mathrm{ALG}(\mathbb{S}, mem)$ that takes as the input $\Robot$'s current sensing data $\Sense{V}{\Position}$ and memory state, and outputs $\Robot$'s next step and memory state.

When two vertices $p$ and $\Position'$ have $\Sense{V}{\Position} = \Sense{V}{\Position'}$, they are indistinguishable from $\Robot$'s perspective and thus must be treated identically by its traversal algorithm.  Figures \ref{fig:sensingregions2D} and \ref{fig:sensingregions3D} illustrate, in 2- and 3-dimensional grid graphs respectively,  regions of vertices for which $\Sense{V}{\Position}$ is identical given sensing range $V=1\ldots 3$. The closer $\Robot$ is to a boundary of $\Dcube$ the smaller the set $\Sense{V}{\Position}$ becomes, because $\Sense{V}{\Position}$ only contains locations inside $\Dcube$. Hence, the closer $\Robot$ is to a boundary, the more information it has about its current position.

The central question we wish to answer is: given sensing range $V$ and $b$ bits of memory, does there exist a \textit{deterministic} function $ALG(\cdot, \cdot)$ that enables $\Robot$ to visit every vertex in $\Dcube$? Formally: 

\begin{definition}
An algorithm $\mathrm{ALG}$ is said to \textbf{patrol} $\mathcal{L} \subset \Dcube$ if, given any initial position $v_1 \in \mathcal{L}$ and initial memory state, executing $\mathrm{ALG}$ causes $\Robot$ to visit all vertices  $v \in \mathcal{L}$ within a finite number of steps.
\label{definition:patrolling}
\end{definition}

Note that \cref{definition:patrolling} implies, in particular, that $\mathrm{ALG}$ will visit each vertex of $\mathcal{L}$ an unlimited number of times (if we let it run indefinitely), hence the term ``patrolling''. We are generally interested in the case $\mathcal{L} = \Dcube$ (i.e., patrolling the entire grid graph), but shall sometimes let $\mathcal{L}$ be a strict subset of $\Dcube$ in proofs. 


\section{Patrolling with 0 Bits of Memory}
\label{section:patrol0bit}

We first consider the problem of visiting every vertex in $\Dcube$ assuming $\Robot$  has sensing range $V$ and $0$ bits of internal state memory. We shall show that, in this setting, whether $\Robot$ is capable of patrolling $\Dcube$ depends on the dimensions of $\Dcube$. Our main result is the following exact characterization of the set of grid graphs that can be patrolled:

\begin{theorem} Let $\Dcube = [n_1] \times [n_2] \times \ldots [n_d]$. There exists an algorithm that patrols $\Dcube$ with $0$ bits of memory and $V$ sensing range if and only if:

\begin{enumerate}
    \item There is at most one index $i$ such that $n_i > 2V + 1$, and
    \item $\prod_{i=1}^{d} \min \big( {n_{i},2V+1} \big)$ is even or equals $1$ (i.e., $\Dcube$ contains a single vertex).
\end{enumerate}
\label{theorem:grid graphsthatcanbepatrolledwithVvisibility0bits}
\end{theorem}

Condition (1) of \cref{theorem:grid graphsthatcanbepatrolledwithVvisibility0bits} says that patrollable $d$-dimensional grids are bounded in all dimensions but at most one, and Condition (2) says they must have an even number of ``sensing regions''. Sensing regions - to be formally defined later in this section - are regions of vertices that are indistinguishable to $\Robot$ (see Figures \ref{fig:sensingregions2D} and \ref{fig:sensingregions3D}). The idea behind the proof of \cref{theorem:grid graphsthatcanbepatrolledwithVvisibility0bits} is to partition $\Dcube$ into sensing regions and to show that if there is an odd number of such regions, or if any such region contains a $2 \times 2$ subregion, $\Dcube$ cannot be patrolled. Vice versa, when the conditions of \cref{theorem:grid graphsthatcanbepatrolledwithVvisibility0bits} hold, we shall describe explicit algorithms for patrolling $\Dcube$: \ref{alg:V=1 b=0-Patrolling Algorithm} and \ref{alg:Vmorethan1 m=0-Patrolling Algorithm}. 

One caveat regarding Algorithms \ref{alg:V=1 b=0-Patrolling Algorithm} and \ref{alg:Vmorethan1 m=0-Patrolling Algorithm} is that they require the dimensions of $\Dcube$ to be known to $\Robot$ in advance. This does not falsify \cref{theorem:grid graphsthatcanbepatrolledwithVvisibility0bits}, since the Theorem is concerned with whether an $0$-bit algorithm exists for patrolling a \textit{given} grid graph, and we may embed information about said grid graph in the algorithm. However, it is desirable to find algorithms that can patrol $\Dcube$ without knowing its dimensions. Later, in \cref{section:patrol1bit}, we shall give an algorithm for patrolling any arbitrary $d$-dimensional grid graph using $1$ bit of memory that works generically, without requiring knowledge of $\Dcube$'s dimensions in advance.

Before proving \cref{theorem:grid graphsthatcanbepatrolledwithVvisibility0bits}, let us build some intuition about patrolling with $0$ bits of memory. In particular, let us show that any algorithm that patrols the grid $\Dcube$ also finds a Hamiltonian cycle of $\Dcube$.

\begin{definition}
    Let $\mathrm{ALG}$ be an algorithm that patrols $\Dcube$.  Suppose $\Robot$ executes $\mathrm{ALG}$ in $\Dcube$ starting from position $v_1$, and let $v_i$ be the vertex of $\Robot$ after $i$ steps. The \textbf{walk from $v_1$ induced by $\mathrm{ALG}$}, denoted $\mathcal{W}_{\mathrm{ALG}}(v_1)$, is the walk
    $v_1 \ldots v_T$ where $T$ indicates the first step at which $\Robot$ has visited all vertices of  $\Dcube$. The \textbf{length  of $\mathcal{W}_{\mathrm{ALG}}(v_1)$} is $T$.
\label{definition:walkinducedbyalg}
\end{definition}

\begin{lemma}
Suppose $\mathrm{ALG}$ is an algorithm that patrols $\Dcube$ using $0$ bits of memory and sensing range $V$. Let $\mathcal{W}_{\mathrm{ALG}}(v_1) = v_1 \ldots v_T$. Then $v_1 \ldots v_T v_1$ is a Hamiltonian cycle of $\Dcube$, and $T = \prod_{i=1}^{d} n_i$.
\label{lemma:coveringispatrollingwhenV=0}
\end{lemma}


\begin{proof}
Since $\mathrm{ALG}$ uses $0$ bits of internal memory, it determines $\Robot$'s next step \textit{uniquely} based on the vertex $\Robot$ is currently located at. Let $\Dcube_{\mathrm{ALG}}$ be the directed graph whose vertices are those of $\Dcube$ and where there is an edge $(u,v)$ if, whenever $\Robot$ is at $v$, $\mathrm{ALG}$ causes it to step to $u$. We note two facts: (i) The out-degree of any vertex in $\Dcube_{\mathrm{ALG}}$ is  $1$. (ii) Since $\mathrm{ALG}$ patrols $\Dcube$, there is a directed path in $\Dcube_{\mathrm{ALG}}$ from any vertex $v$ in $\Dcube$ to any other vertex $u$ (including from $v$ to itself). From (i) and (ii) it necessarily follows that $\Dcube_{\mathrm{ALG}}$ is a directed cycle spanning all vertices of $\Dcube$. Hence $v_1 \ldots v_T v_1$ is a Hamiltonian cycle of $\Dcube$ and $T = \prod_{i=1}^{d} n_i$.
\end{proof}

Given some patrolling algorithm $\mathrm{ALG}$, we shall refer to $v_1 \ldots v_T v_1$ as the \textbf{Hamiltonian cycle from $v_1$ induced by $\mathrm{ALG}$}, denoted $\mathcal{C}_{\mathrm{ALG}}(v_1)$.

\cref{lemma:coveringispatrollingwhenV=0} says that any algorithm that visits all vertices of $\Dcube$ with $0$ bits of memory must repeatedly traverse a Hamiltonian cycle of $\Dcube$. As a preliminary result, this fact immediately allows us to prove a special case of \cref{theorem:grid graphsthatcanbepatrolledwithVvisibility0bits}. Let us recall the following well-known fact (see, e.g., \cite{moore2011nature}):

\begin{lemma}
Let $\Dcube = [n_1] \times [n_2] \times \ldots [n_d]$. If $\prod_{i=1}^{d} n_{i}$ is odd, no Hamiltonian cycle of $\Dcube$ exists \cite{moore2011nature}.
\label{lemma:hamiltoniancyclerequiresevenproduct}
\end{lemma}

\begin{proof}
    Color a vertex $(x_1,x_2 \ldots x_d)$ of $\Dcube$ ``red'' if the sum of its coordinates is even, and ``blue'' otherwise. Since every edge of $\Dcube$ connects vertices of different colors, and $\prod_{i=1}^{d} n_{i}$ is odd, a Hamiltonian path of $\Dcube$ must start and end at vertices of the same color. These two vertices cannot be connected by an edge, and so it is impossible to form a Hamiltonian cycle of $\Dcube$.
\end{proof}

\begin{corollary}[Special case of \cref{theorem:grid graphsthatcanbepatrolledwithVvisibility0bits}]
Let $\Dcube = [n_1] \times [n_2] \times \ldots [n_d]$. If $\prod_{i=1}^{d} n_{i}$ is odd and greater than $1$, no algorithm exists that can patrol $\Dcube$ with $0$ bits of memory.
\label{corollary:nopatrolwhen0bitsandoddproductofdimensions}
\end{corollary}

\begin{proof}
Assume for contradiction some algorithm patrols $\Dcube$ using $0$ bits of memory. Then \cref{lemma:coveringispatrollingwhenV=0} says a Hamiltonian cycle of $\Dcube$ exists, contradicting \cref{lemma:hamiltoniancyclerequiresevenproduct}.
\end{proof}

Suppose $\Robot$ is located at $\Position = (x_1, x_2 \ldots, x_d)$. We need a convenient way to represent what $\Robot$ senses along some dimension $i$ at position $\Position$. To this end, we project  $\Sense{V}{\Position}$ onto dimension $i$ and measure the distance to the boundaries of $\Dcube$:

\begin{definition}
We define a ``projection'' of $\Sense{V}{\Position}$ onto dimension $i$ as follows:

\begin{equation}
    \Projection{i}{V}{\Position} \defeq \Set{\Position' - \Position \in \Sense{V}{\Position}}{\Position' \in \Dcube, \Position' - \Position = (0,\ldots 0, x_i' - x_i, 0, \ldots 0)}
\end{equation}

We further define the \textbf{boundary distances from $\Position$ along dimension $i$} as $(l_i, r_i)$ where $l_i = |\Set{\Position' - \Position \in \Projection{i}{V}{\Position}}{p' \in \Dcube, x_i' < x_i}|$ and $r_i = |\Set{\Position' - \Position \in \Projection{i}{V}{\Position}}{p' \in \Dcube, x_i' > x_i}|$.

\label{definition:projection_li_ri}
\end{definition}

$l_i$ and $r_i$ are the number of negative and positive unit steps (respectively) that $\Robot$ can move in dimension $i$ before hitting a boundary of $\Dcube$, but they are bounded above by $V$. When $l_i = r_i = V$, $\Robot$ cannot tell how far it is from a boundary of $\Dcube$ along dimension $i$. When $l_i < V$ or $r_i < V$, however, $\Robot$ can sense how close it is to a boundary of $\Dcube$ along dimension $i$, which helps localization. 

Because $\Robot$ has a sensing range of $V$, when $\Dcube$ is large there will inevitably exist indistinguishable vertices $p$ and $\Position'$ for which $\Sense{V}{\Position} = \Sense{V}{\Position'}$. Let us define an equivalence relation on such vertices:

\begin{definition}
    Let $\Position, \Position' \in \Dcube$. We define the relation $\Position \sim \Position'$ if $\Sense{V}{\Position} = \Sense{V}{\Position'}$. The equivalence class of $\Position$, denoted $\tilde{\Position} \defeq \Set{\Position'}{\Position' \sim \Position}$, is called \textbf{the sensing region induced by $\Position$}.
    \label{definition:sensingregions}
\end{definition}

\begin{definition}
    Let $\SenseGraph{\Dcube}{V}$ be the graph whose vertices are the different sensing regions of $\Dcube$ given sensing range $V$, and where there is an edge $(\tilde{\Position}, \tilde{\Position'})$ iff there exist vertices $v \in \tilde{\Position}$ and $v' \in \tilde{\Position'}$, where $\tilde{\Position}$ and $\tilde{\Position'}$ are distinct sensing regions,  such that $(v,v')$ is an edge of $\Dcube$. $\SenseGraph{\Dcube}{V}$ is called \textbf{the sensing region graph of $\Dcube$ given visibility $V$}.
\end{definition}

\begin{figure}
    \centering
    \begin{subfigure}{0.26\textwidth}
        \centering
        \pgfmathsetmacro{\VisibilityRange}{1}
        \resizebox{0.99\textwidth}{!}{
            \begin{tikzpicture}
    \pgfmathsetmacro{\CoreWidth}{2 * \VisibilityRange + 1}
    \pgfmathsetmacro{\RectSize}{2 * \VisibilityRange + \CoreWidth}

    \begin{scope}[opacity=0.6]
        \foreach \i in {1, ..., \VisibilityRange} {
            \pgfmathsetmacro{\grayvalue}{\i / (\VisibilityRange + 5)}
            \definecolor{fillcolor}{gray}{\grayvalue}
            \begin{scope}[fill=fillcolor]
                \fill (0, \i - 1) rectangle (\RectSize, \i);
                \fill (0, \RectSize - \i) rectangle (\RectSize, \RectSize - \i + 1);
                \fill (\i - 1, 0) rectangle (\i, \RectSize);
                \fill (\RectSize - \i, 0) rectangle (\RectSize - \i + 1, \RectSize);
            \end{scope}
        }
        \fill[black!2] (0, \VisibilityRange) rectangle (\RectSize, \RectSize - \VisibilityRange);
        \fill[black!2] (\VisibilityRange, 0) rectangle (\RectSize - \VisibilityRange, \RectSize);
    \end{scope}

    \draw[opacity = 0.4, dashed] (0, 0) grid (\RectSize, \RectSize);
    \draw[line width=2pt] (0, 0) rectangle (\RectSize, \RectSize);
    \begin{scope}[line width = 1pt]
        \foreach \i in {1, ..., \VisibilityRange} {
            \draw (0, \i) -- (\RectSize, \i);
            \draw (0, \RectSize - \i) -- (\RectSize, \RectSize - \i);
            \draw (\i, 0) -- (\i, \RectSize);
            \draw (\RectSize - \i, 0) -- (\RectSize - \i, \RectSize);
        }
    \end{scope}

    \begin{scope}[shift={(0.5, 0.5)}]
        \begin{scope}[white]
            \foreach \row in {1, ..., \VisibilityRange} {
                \foreach \col in {1, ..., \VisibilityRange} {
                    \pgfmathtruncatemacro{\rowstart}{(\row - 1) * (2 * \VisibilityRange + 1)}
                    \pgfmathtruncatemacro{\index}{\rowstart + \col}
                    \node at (\col - 1, 3 * \VisibilityRange + 1 + \VisibilityRange - \row) {$\index$};
    
                    \pgfmathtruncatemacro{\index}{\rowstart + \VisibilityRange + 1 + \col}
                    \node at (3 * \VisibilityRange + 1 + \col - 1, 3 * \VisibilityRange + 1 + \VisibilityRange - \row) {$\index$};
    
                    \pgfmathtruncatemacro{\rowstart}{(\VisibilityRange + \row) * (2 * \VisibilityRange + 1)}
                    \pgfmathtruncatemacro{\index}{\rowstart + \col}
                    \node at (\col - 1, \VisibilityRange - \row) {$\index$};
    
                    \pgfmathtruncatemacro{\index}{\rowstart + \VisibilityRange + 1 + \col}
                    \node at (3 * \VisibilityRange + 1 + \col - 1, \VisibilityRange - \row) {$\index$};
                }
            }
        \end{scope}
        \begin{scope}[black]
            \foreach \row in {1, ..., \VisibilityRange} {
                \pgfmathtruncatemacro{\rowstart}{(\row - 1) * (2 * \VisibilityRange + 1)}
                \pgfmathtruncatemacro{\index}{\rowstart + \VisibilityRange + 1}
                \node at (2 * \VisibilityRange, 3 * \VisibilityRange + 1 + \VisibilityRange - \row) {$\index$};

                \pgfmathtruncatemacro{\rowstart}{(\VisibilityRange + \row) * (2 * \VisibilityRange + 1)}
                \pgfmathtruncatemacro{\index}{\rowstart + \VisibilityRange + 1}
                \node at (2 * \VisibilityRange, \VisibilityRange - \row) {$\index$};
            }
            \pgfmathtruncatemacro{\rowstart}{(\VisibilityRange) * (2 * \VisibilityRange + 1)}
            \foreach \col in {1, ..., \VisibilityRange} {
                \pgfmathtruncatemacro{\index}{\rowstart + \col}
                \node at (\col - 1, 2 * \VisibilityRange) {$\index$};

                \pgfmathtruncatemacro{\index}{\rowstart + \VisibilityRange + 1 + \col}
                \node at (3 * \VisibilityRange + \col, 2 * \VisibilityRange) {$\index$};
            }            
            \pgfmathtruncatemacro{\index}{\rowstart + \VisibilityRange + 1}
            \node at (2 * \VisibilityRange, 2 * \VisibilityRange) {$\index$};
        \end{scope}
    \end{scope}
\end{tikzpicture}
        }
        \caption{$V = 1$}
    \end{subfigure}%
    \hfill
    \begin{subfigure}{0.25\textwidth}
        \centering
        \pgfmathsetmacro{\VisibilityRange}{2}
        \resizebox{\textwidth}{!}{
            \begin{tikzpicture}
    \pgfmathsetmacro{\CoreWidth}{2 * \VisibilityRange + 1}
    \pgfmathsetmacro{\RectSize}{2 * \VisibilityRange + \CoreWidth}

    \begin{scope}[opacity=0.6]
        \foreach \i in {1, ..., \VisibilityRange} {
            \pgfmathsetmacro{\grayvalue}{\i / (\VisibilityRange + 5)}
            \definecolor{fillcolor}{gray}{\grayvalue}
            \begin{scope}[fill=fillcolor]
                \fill (0, \i - 1) rectangle (\RectSize, \i);
                \fill (0, \RectSize - \i) rectangle (\RectSize, \RectSize - \i + 1);
                \fill (\i - 1, 0) rectangle (\i, \RectSize);
                \fill (\RectSize - \i, 0) rectangle (\RectSize - \i + 1, \RectSize);
            \end{scope}
        }
        \fill[black!2] (0, \VisibilityRange) rectangle (\RectSize, \RectSize - \VisibilityRange);
        \fill[black!2] (\VisibilityRange, 0) rectangle (\RectSize - \VisibilityRange, \RectSize);
    \end{scope}

    \draw[opacity = 0.4, dashed] (0, 0) grid (\RectSize, \RectSize);
    \draw[line width=2pt] (0, 0) rectangle (\RectSize, \RectSize);
    \begin{scope}[line width = 1pt]
        \foreach \i in {1, ..., \VisibilityRange} {
            \draw (0, \i) -- (\RectSize, \i);
            \draw (0, \RectSize - \i) -- (\RectSize, \RectSize - \i);
            \draw (\i, 0) -- (\i, \RectSize);
            \draw (\RectSize - \i, 0) -- (\RectSize - \i, \RectSize);
        }
    \end{scope}

    \begin{scope}[shift={(0.5, 0.5)}]
        \begin{scope}[white]
            \foreach \row in {1, ..., \VisibilityRange} {
                \foreach \col in {1, ..., \VisibilityRange} {
                    \pgfmathtruncatemacro{\rowstart}{(\row - 1) * (2 * \VisibilityRange + 1)}
                    \pgfmathtruncatemacro{\index}{\rowstart + \col}
                    \node at (\col - 1, 3 * \VisibilityRange + 1 + \VisibilityRange - \row) {$\index$};
    
                    \pgfmathtruncatemacro{\index}{\rowstart + \VisibilityRange + 1 + \col}
                    \node at (3 * \VisibilityRange + 1 + \col - 1, 3 * \VisibilityRange + 1 + \VisibilityRange - \row) {$\index$};
    
                    \pgfmathtruncatemacro{\rowstart}{(\VisibilityRange + \row) * (2 * \VisibilityRange + 1)}
                    \pgfmathtruncatemacro{\index}{\rowstart + \col}
                    \node at (\col - 1, \VisibilityRange - \row) {$\index$};
    
                    \pgfmathtruncatemacro{\index}{\rowstart + \VisibilityRange + 1 + \col}
                    \node at (3 * \VisibilityRange + 1 + \col - 1, \VisibilityRange - \row) {$\index$};
                }
            }
        \end{scope}
        \begin{scope}[black]
            \foreach \row in {1, ..., \VisibilityRange} {
                \pgfmathtruncatemacro{\rowstart}{(\row - 1) * (2 * \VisibilityRange + 1)}
                \pgfmathtruncatemacro{\index}{\rowstart + \VisibilityRange + 1}
                \node at (2 * \VisibilityRange, 3 * \VisibilityRange + 1 + \VisibilityRange - \row) {$\index$};

                \pgfmathtruncatemacro{\rowstart}{(\VisibilityRange + \row) * (2 * \VisibilityRange + 1)}
                \pgfmathtruncatemacro{\index}{\rowstart + \VisibilityRange + 1}
                \node at (2 * \VisibilityRange, \VisibilityRange - \row) {$\index$};
            }
            \pgfmathtruncatemacro{\rowstart}{(\VisibilityRange) * (2 * \VisibilityRange + 1)}
            \foreach \col in {1, ..., \VisibilityRange} {
                \pgfmathtruncatemacro{\index}{\rowstart + \col}
                \node at (\col - 1, 2 * \VisibilityRange) {$\index$};

                \pgfmathtruncatemacro{\index}{\rowstart + \VisibilityRange + 1 + \col}
                \node at (3 * \VisibilityRange + \col, 2 * \VisibilityRange) {$\index$};
            }            
            \pgfmathtruncatemacro{\index}{\rowstart + \VisibilityRange + 1}
            \node at (2 * \VisibilityRange, 2 * \VisibilityRange) {$\index$};
        \end{scope}
    \end{scope}
\end{tikzpicture}
        }
        \caption{$V = 2$}
    \end{subfigure}%
    \hfill
    \begin{subfigure}{0.25\textwidth}
        \centering
        \pgfmathsetmacro{\VisibilityRange}{3}
        \resizebox{0.99\textwidth}{!}{
            \begin{tikzpicture}
    \pgfmathsetmacro{\CoreWidth}{2 * \VisibilityRange + 1}
    \pgfmathsetmacro{\RectSize}{2 * \VisibilityRange + \CoreWidth}

    \begin{scope}[opacity=0.6]
        \foreach \i in {1, ..., \VisibilityRange} {
            \pgfmathsetmacro{\grayvalue}{\i / (\VisibilityRange + 5)}
            \definecolor{fillcolor}{gray}{\grayvalue}
            \begin{scope}[fill=fillcolor]
                \fill (0, \i - 1) rectangle (\RectSize, \i);
                \fill (0, \RectSize - \i) rectangle (\RectSize, \RectSize - \i + 1);
                \fill (\i - 1, 0) rectangle (\i, \RectSize);
                \fill (\RectSize - \i, 0) rectangle (\RectSize - \i + 1, \RectSize);
            \end{scope}
        }
        \fill[black!2] (0, \VisibilityRange) rectangle (\RectSize, \RectSize - \VisibilityRange);
        \fill[black!2] (\VisibilityRange, 0) rectangle (\RectSize - \VisibilityRange, \RectSize);
    \end{scope}

    \draw[opacity = 0.4, dashed] (0, 0) grid (\RectSize, \RectSize);
    \draw[line width=2pt] (0, 0) rectangle (\RectSize, \RectSize);
    \begin{scope}[line width = 1pt]
        \foreach \i in {1, ..., \VisibilityRange} {
            \draw (0, \i) -- (\RectSize, \i);
            \draw (0, \RectSize - \i) -- (\RectSize, \RectSize - \i);
            \draw (\i, 0) -- (\i, \RectSize);
            \draw (\RectSize - \i, 0) -- (\RectSize - \i, \RectSize);
        }
    \end{scope}

    \begin{scope}[shift={(0.5, 0.5)}]
        \begin{scope}[white]
            \foreach \row in {1, ..., \VisibilityRange} {
                \foreach \col in {1, ..., \VisibilityRange} {
                    \pgfmathtruncatemacro{\rowstart}{(\row - 1) * (2 * \VisibilityRange + 1)}
                    \pgfmathtruncatemacro{\index}{\rowstart + \col}
                    \node at (\col - 1, 3 * \VisibilityRange + 1 + \VisibilityRange - \row) {$\index$};
    
                    \pgfmathtruncatemacro{\index}{\rowstart + \VisibilityRange + 1 + \col}
                    \node at (3 * \VisibilityRange + 1 + \col - 1, 3 * \VisibilityRange + 1 + \VisibilityRange - \row) {$\index$};
    
                    \pgfmathtruncatemacro{\rowstart}{(\VisibilityRange + \row) * (2 * \VisibilityRange + 1)}
                    \pgfmathtruncatemacro{\index}{\rowstart + \col}
                    \node at (\col - 1, \VisibilityRange - \row) {$\index$};
    
                    \pgfmathtruncatemacro{\index}{\rowstart + \VisibilityRange + 1 + \col}
                    \node at (3 * \VisibilityRange + 1 + \col - 1, \VisibilityRange - \row) {$\index$};
                }
            }
        \end{scope}
        \begin{scope}[black]
            \foreach \row in {1, ..., \VisibilityRange} {
                \pgfmathtruncatemacro{\rowstart}{(\row - 1) * (2 * \VisibilityRange + 1)}
                \pgfmathtruncatemacro{\index}{\rowstart + \VisibilityRange + 1}
                \node at (2 * \VisibilityRange, 3 * \VisibilityRange + 1 + \VisibilityRange - \row) {$\index$};

                \pgfmathtruncatemacro{\rowstart}{(\VisibilityRange + \row) * (2 * \VisibilityRange + 1)}
                \pgfmathtruncatemacro{\index}{\rowstart + \VisibilityRange + 1}
                \node at (2 * \VisibilityRange, \VisibilityRange - \row) {$\index$};
            }
            \pgfmathtruncatemacro{\rowstart}{(\VisibilityRange) * (2 * \VisibilityRange + 1)}
            \foreach \col in {1, ..., \VisibilityRange} {
                \pgfmathtruncatemacro{\index}{\rowstart + \col}
                \node at (\col - 1, 2 * \VisibilityRange) {$\index$};

                \pgfmathtruncatemacro{\index}{\rowstart + \VisibilityRange + 1 + \col}
                \node at (3 * \VisibilityRange + \col, 2 * \VisibilityRange) {$\index$};
            }            
            \pgfmathtruncatemacro{\index}{\rowstart + \VisibilityRange + 1}
            \node at (2 * \VisibilityRange, 2 * \VisibilityRange) {$\index$};
        \end{scope}
    \end{scope}
\end{tikzpicture}
        }
        \caption{$V = 3$}
    \end{subfigure}
    \caption{The sensing regions on $2$-dimensional grid graphs given sensing range $V = 1$, $2$, or $3$. Regions are labelled from $1$ to $(2V + 1)^2$.}
    \label{fig:sensingregions2D}
\end{figure}
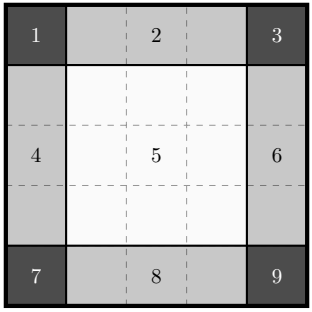
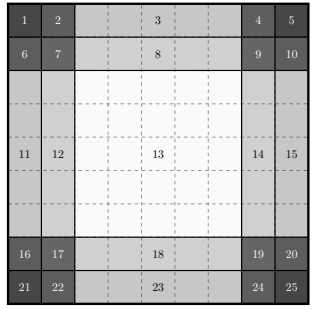
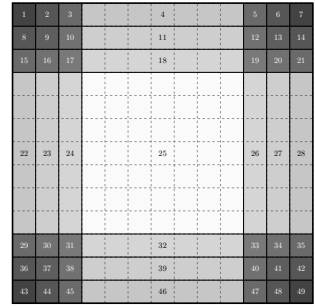

\begin{figure}
    \centering
    \begin{subfigure}{0.25\textwidth}
        \centering
        \pgfmathsetmacro{\VisibilityRange}{1}
        \resizebox{1\textwidth}{!}{
            \subimport{images/}{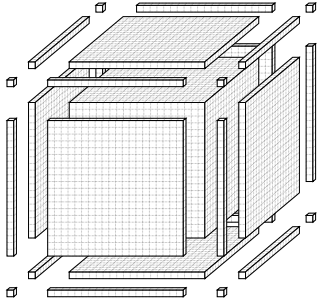}
        }
        \caption{$V = 1$}
    \end{subfigure}%
    \hfill
    \begin{subfigure}{0.25\textwidth}
        \centering
        \pgfmathsetmacro{\VisibilityRange}{2}
        \resizebox{\textwidth}{!}{
            \subimport{images/}{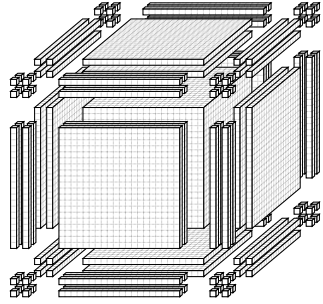}
        }
        \caption{$V = 2$}
    \end{subfigure}%
    \hfill
    \begin{subfigure}{0.25\textwidth}
        \centering
        \pgfmathsetmacro{\VisibilityRange}{3}
        \resizebox{0.99\textwidth}{!}{
            \subimport{images/}{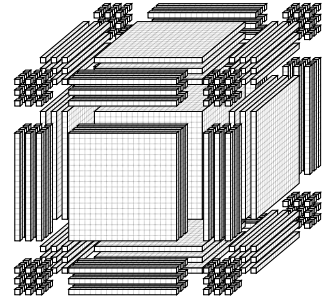}
        }
        \caption{$V = 3$}
    \end{subfigure}
    \caption{The sensing regions of $3$-dimensional grid graphs given sensing range $V = 1$, $2$, or $3$.}
    \label{fig:sensingregions3D}
\end{figure}

The sensing regions of 2- and 3-dimensional grid graphs are illustrated in \cref{fig:sensingregions2D} and \cref{fig:sensingregions3D}, respectively. We note two straightforward facts about sensing regions:

\begin{enumerate}[label=(\Alph*)]
    \item Let $\mathrm{ALG}$ be $\Robot$'s traversal algorithm. As long as $\Robot$ is located within a given sensing region $\tilde{\Position}$, $\mathrm{ALG}$ can only move $\Robot$ in one direction.
    \item If $\Dcube = [n_1] \times [n_2] \times \ldots [n_d]$, then $\SenseGraph{\Dcube}{V}$ is isomorphic to the $d$-dimensional grid  graph $[m_1] \times [m_2] \times \ldots \times [m_d]$ where $m_i =\min \big( {n_{i},2V+1} \big)$. 
\end{enumerate}

(A) follows from the fact that vertices in $\tilde{\Position}$ are indistinguishable to $\Robot$, and $\Robot$ has $0$ bits of memory, hence $\mathrm{ALG}$ must always make the same decision as long as $\Robot$ is in $\tilde{\Position}$. 

(B) is established by noting that two locations $p,p' \in \Dcube$ belong to the same sensing region if and only if $p$ and $p'$ have the same boundary distances $(l_1, r_1), \ldots (l_d, r_d)$ along each dimension (\cref{definition:projection_li_ri}). The set of possible boundary distances $(l_i, r_i)$ is of size $m_i$, because $\Robot$ can sense vertices at distance less than or equal to $V$ along dimension $i$. Each sensing region is connected to adjacent sensing regions in the same way grid vertices are (see Figures \ref{fig:sensingregions2D}, \ref{fig:sensingregions3D}), making  $\SenseGraph{\Dcube}{V}$  isomorphic to $[m_1] \times [m_2] \times \ldots [m_d]$.

We can now establish some facts about whether a given $d$-dimensional grid is patrollable:

\begin{lemma}
     Let $\Dcube = [n_1] \times [n_2] \times \ldots [n_d]$. If $\prod_{i=1}^{d} m_i$ (where $m_i =\min \big( {n_{i},2V+1} \big)$) is odd and greater than $1$, no algorithm exists that can patrol $\Dcube$ with $0$ bits of memory.

\label{lemma:nopatroloddnumberofsensingregions}
\end{lemma}

\begin{proof}
    Suppose for contradiction that such an algorithm, $\mathrm{ALG}$, exists and let $\mathcal{C}_{\mathrm{ALG}}(v_1) = v_1 \ldots v_T v_{1}$ be the Hamiltonian cycle from $v_1$ induced by $\mathrm{ALG}$. Let $\tilde{\mathcal{C}}$ be the walk  $\tilde{v_1}\ldots \tilde{v_T} \tilde{v_{1}}$ in $\SenseGraph{\Dcube}{V}$ after deleting duplicates (that is, vertices $\tilde{v_i}$ for which $\tilde{v_i} = \tilde{v_{i-1}}$).  $\tilde{\mathcal{C}}$ is a cycle that traverses all of  $\SenseGraph{\Dcube}{V}$. In fact, $\tilde{\mathcal{C}}$ is a Hamiltonian cycle, since by (A), $\Robot$ cannot visit the same sensing region twice before visiting all others (by the same proof as \cref{lemma:coveringispatrollingwhenV=0}). But $\SenseGraph{\Dcube}{V}$ is isomorphic to $[m_1] \times [m_2] \times \ldots [m_d]$, contradicting  \cref{lemma:hamiltoniancyclerequiresevenproduct}.
\end{proof}

\begin{lemma}
     Let $\Dcube = [n_1] \times [n_2] \times \ldots [n_d]$. If there exist different $i,j$ such that  $n_i, n_j > 2V + 1$, no algorithm exists that can patrol $\Dcube$ with $0$ bits of memory.

\label{lemma:nopatrol2x2sensingsubregion}
\end{lemma}

\begin{proof}
Suppose without loss of generality that $i=1, j=2$. We also suppose for contradiction that an algorithm $\mathrm{ALG}$ exists that patrols $\Dcube$ with $0$ bits of memory.

Since $n_1, n_2 > 2V+1$, the sensing region $\tilde{r} \in \SenseGraph{\Dcube}{V}$ that contains $(V+1,V+1,1,\ldots 1) \in \Dcube$ necessarily also contains  $(V+1,V+2,1,\ldots 1)$, $(V+2,V+1,1,\ldots 1)$ and $(V+2,V+2,1,\ldots 1)$. Let us denote these four vertices $p_1, p_2, p_3, p_4$ respectively. Let $\mathcal{C}_{\mathrm{ALG}}(v_1) = v_1 \ldots v_T v_{1}$ and  $\tilde{\mathcal{C}}$ be as in the proof of \cref{{lemma:nopatroloddnumberofsensingregions}}. Since $\tilde{\mathcal{C}}$ is a Hamiltonian cycle of $\SenseGraph{\Dcube}{V}$, we know that the Hamiltonian path $v_1 \ldots v_T$ of $\Dcube$ enters each sensing region exactly once. Furthermore, by (A), $\Robot$ can only move in one direction within a given sensing region. However, the sensing region $\tilde{r}$ contains the $2 \times 2$ subregion $\{p_1, p_2, p_3, p_4\}$, hence requires $\Robot$ to move in two directions (along the $x_1$ and $x_2$ axes) to visit all its vertices. Consequently, it cannot be fully visited by the path $v_1 \ldots v_T$ - contradiction. \end{proof}

Lemmas \ref{lemma:nopatroloddnumberofsensingregions} and \ref{lemma:nopatrol2x2sensingsubregion} establish that if the conditions of \cref{theorem:grid graphsthatcanbepatrolledwithVvisibility0bits} are unfulfilled, no algorithm exists that can patrol $\Dcube$ with $0$ bits of memory. It remains to show that if the conditions are fulfilled, such algorithms exist.

We proceed by explicitly describing such algorithms. Specifically, \cref{theorem:grid graphsthatcanbepatrolledwithVvisibility0bits} requires us to construct, for every $V$, an algorithm that patrols a grid graph $\Dcube$ satisfying:

\begin{enumerate}
    \item There is at most one index $i$ such that $n_i > 2V + 1$, and
    \item $\prod_{i=1}^{d} \min \big( {n_{i},2V+1} \big)$ is even or equals $1$.
\end{enumerate}

Note that this problem is not strictly equivalent to finding \textit{any} Hamiltonian cycle of the sensing region graph, because sensing regions may contain more than $1$ vertex, and an arbitrary Hamiltonian cycle may require the robot to move in several different directions inside the same sensing region. Hence, we need to identify a particular Hamiltonian cycle that works. 

We give two different algorithms for the cases $V=1$ and $V > 1$ (the algorithm for $V > 1$ constructs a different kind of Hamiltonian cycle and must be handled separately). Both algorithms assume, without loss of generality, that $n_1 \geq n_2 \geq \ldots n_d$.

\subsubsection*{$V=1$, 0-bit Patrolling Algorithm}

\cref{theorem:grid graphsthatcanbepatrolledwithVvisibility0bits} allows for exactly one of the dimensions of $\Dcube$ to be greater than $2V+1$. Hence, we assume w.l.o.g. that $n_2, n_3 \ldots n_d \leq 2V+1$ (so $x_1$ is the unique dimension that is allowed to be ``large'').

\SetKwFunction{Parity}{Parity}
\SetKwFunction{movedown}{down}
Our algorithm for the $V=1$ case, \cref{alg:V=1 b=0-Patrolling Algorithm}, is straightforward. Its pseudocode returns a $step$ vector based on $\Robot$'s sensing data, which $\Robot$ adds to its current location's coordinates so as to determine its next location. It works by patrolling $\Dcube$ in a ``zig-zag'' fashion, moving in a fixed direction along the $x_1$ axis until it hits a boundary, and subsequently flipping direction and incrementing or decrementing $x_2$; moving in a fixed direction along $x_2$ until it hits a boundary and subsequently flipping direction and incrementing or decrementing $x_3$; and so on. A subroutine called $Parity$ (\cref{alg:parity}) is used to keep track of whether $\Robot$ needs to increment or decrement its coordinate along a given axis via a variable called  $\movedown$. The behavior of \cref{alg:V=1 b=0-Patrolling Algorithm} over $\Dcube = [5] \times [3] \times [3] \times [2]$ is illustrated in \cref{fig:visualization.v=1.m=0}.  

\begin{algorithm}
    
    \SetKwProg{Fn}{Function}{}{}
    \DontPrintSemicolon
    
    \Fn{\Parity}{
    \KwIn{$V ,n, (\textit{l}, \textit{r})$}
        \uIf{\textit{l} < $V$} {
            \Return $\textit{l} \mod 2$\;
        }
        \Return $n - \textit{r} + 1 \mod 2$\;
    }
    
\caption{The sub-routine ``\textbf{Parity}'' used in \cref{alg:V=1 b=0-Patrolling Algorithm}.}
\label{alg:parity}	
\end{algorithm}

\begin{algorithm}
    \SetKwFunction{step}{step}
    \SetKwFunction{Parity}{Parity}
    \SetKwFunction{MakeMove}{MakeMoveMemoryless}

    \SetKw{And}{and}
    \SetKwProg{Fn}{Function}{}{}    
    \DontPrintSemicolon

    \BlankLine

    \Fn{\MakeMove}{
            \KwIn{$V = 1$ - agent sensing range.}
       \KwIn{$n_1, \ldots, n_d$ - the dimensions of $\Dcube$}
        \KwIn{$\{(\textit{l}_1, \textit{r}_1), \ldots (\textit{l}_d, \textit{r}_d)\}$ - see \cref{definition:projection_li_ri}}
        \BlankLine
        $\step = (0, \ldots, 0)$\tcp*{$d$ - dimensional zero vector}
        \BlankLine
        \For{$j = 1$ \KwTo $d$} {
            $\movedown = \sum\limits_{i > j}^{d} \Parity(V, n_i, (\textit{l}_i, \textit{r}_i)) \mod 2$\tcp*{$0$ when $j = d$}
        
            \BlankLine
            \If{$\movedown = 0$ \And $\textit{r}_j > 0$}{
                $\step[j] = 1$\tcp*{move up in the $j$th dimension}
                \Return \step\;
            }
            \BlankLine
            \If{$\movedown = 1$ \And $\textit{l}_j > 0$}{
                $\step[j] = -1$\tcp*{move down in the $j$th dimension}
                \Return \step\;
            }
        }
        \BlankLine
        $\step[d] = -1$\tcp*{move down in the last dimension}
        \Return \step\;
    }
\caption{A memoryless, $V=1$ sensing range patrolling algorithm for $d$-dimensional grids that adhere to the constraints of \cref{theorem:grid graphsthatcanbepatrolledwithVvisibility0bits}.}
\label{alg:V=1 b=0-Patrolling Algorithm}
\end{algorithm}

\begin{figure}
    \centering
        \centering
        \resizebox{0.8\textwidth}{!}{
            \includegraphics[page=1]{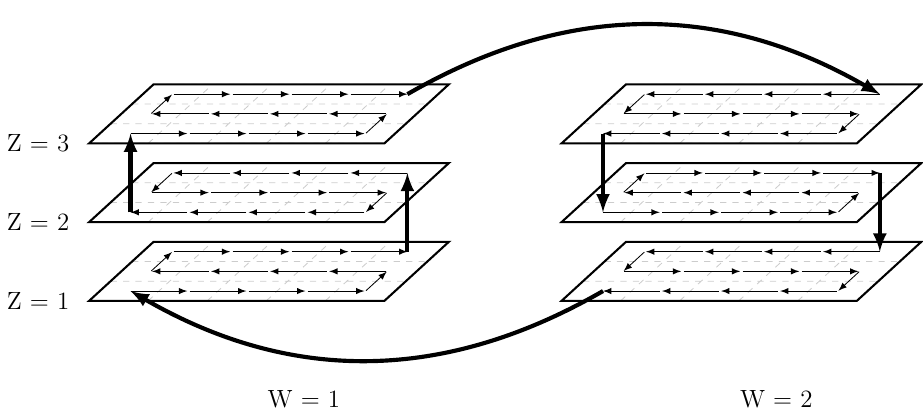}
        }
    \caption{The behavior of \cref{alg:V=1 b=0-Patrolling Algorithm} over $\Dcube = [5] \times [3] \times [3] \times [2]$. Arrows denote the next step of $\Robot$ from a given location.}
    \label{fig:visualization.v=1.m=0}
\end{figure}

Assuming \cref{theorem:grid graphsthatcanbepatrolledwithVvisibility0bits}'s conditions are fulfilled, \cref{alg:V=1 b=0-Patrolling Algorithm} finds a Hamiltonian path of $\Dcube$, and does not violate any of the assumptions in the problem setting:

\begin{itemize}
    \item \textbf{No memory or sensing range constraints are violated:}  As can be seen from the pseudocode, the algorithm only receives as inputs the boundary distances of $\Robot$ and the dimensions of $\Dcube$, using no internal memory. 

    \item \textbf{The algorithm successfully patrols $\Dcube$:} When $V=1$, \cref{theorem:grid graphsthatcanbepatrolledwithVvisibility0bits}'s conditions imply some dimension $n_i$ of $\Dcube$ equals exactly $2$. Due to the parity bit $\movedown$, the path of $\Robot$ is reversed when $x_d = 2$ compared to when $x_d = 1$ (see \cref{fig:visualization.v=1.m=0}). Hence, \cref{alg:V=1 b=0-Patrolling Algorithm} alternates between two opposite Hamiltonian walks, one covering the vertices of $\Dcube$ which have $x_d = 1$ and one covering vertices which have $x_d = 2$. Stitching these walks together forms a Hamiltonian cycle of $\Dcube$.
\end{itemize}

\subsubsection*{$V > 1$, 0-bit Patrolling Algorithm}

Our algorithm for the $V > 1$ case, \cref{alg:Vmorethan1 m=0-Patrolling Algorithm}, needs to patrol a greater variety of grid graphs than the algorithm we gave for the $V = 1$ case, since the conditions asserted by \cref{theorem:grid graphsthatcanbepatrolledwithVvisibility0bits} depend on $V$. As before, we assume w.l.o.g. that $n_2, n_3 \ldots n_d \leq 2V+1$. 

\begin{algorithm}[!ht]
    
    \SetKwFunction{Parity}{Parity}
    \SetKwFunction{step}{step}
    \SetKwFunction{movedown}{down}
    \SetKwFunction{moveup}{up}

    \SetKw{And}{and}
    \SetKw{Continue}{continue}

    \DontPrintSemicolon
    \SetKwProg{Fn}{Function}{}{}    
    \Fn{\MakeMove}{
    \KwIn{As in \cref{alg:V=1 b=0-Patrolling Algorithm}}
        \BlankLine
        $\step = (0, \ldots, 0)$ \tcp*{$d$ - dimensional zero vector}
        \uIf(\tcp*[h]{return back on the $x_1 = 1$ strip}){$\textit{l}_1 = 0$}{
            \For{$j = 2$ \KwTo $d$} {
                $\moveup = \sum\limits_{i > j}^{d} \Parity(V, n_i, (\textit{l}_i, \textit{r}_i)) \mod 2$\tcp*{$0$ when $j = d$}
                \uIf{$\moveup = 0$ \And $\textit{l}_j > 0$}{
                    $\step[j] = -1$\;
                    \Return \step\;
                }
                \uIf{$\moveup = 1$ \And $\textit{r}_j > 0$}{
                    $\step[j] = 1$\;
                    \Return \step\;
                }
            }
            $\step[1] = 1$
        }
        \uElse(\tcp*[h]{execute a zig-zag motion}){
            \For{$j = 1$ \KwTo $d$} {
                $\movedown = \sum\limits_{i > j}^{d} \Parity(V, n_i, (\textit{l}_i, \textit{r}_i)) \mod 2$\tcp*{$0$ when $j = d$}
                \uIf{$\movedown = 0$ \And $\textit{r}_j > 0$}{
                    $\step[j] = 1$\;
                    \Return \step\;
                }
                \uIf{$\movedown = 1$ \And $\textit{l}_j > 0$}{
                    \uIf(\tcp*[f]{skip $x_1 = 1$ strip}){j = 1 \And $\textit{l}_1 = 1$}{
                        \Continue\;
                    }
                    $\step[j] = -1$\;
                    \Return \step\;
                }
            }
            $\step[1] = -1$
            
        }
       \Return \step\;
    }
    
\caption{A memoryless, $V>1$ sensing range patrolling algorithm for $d$-dimensional grids that adhere to the constraints of \cref{theorem:grid graphsthatcanbepatrolledwithVvisibility0bits}.}
\label{alg:Vmorethan1 m=0-Patrolling Algorithm}	
\end{algorithm}

Similar to \cref{alg:V=1 b=0-Patrolling Algorithm}, it can be seen from the pseudocode that \cref{alg:Vmorethan1 m=0-Patrolling Algorithm} is memoryless and does not violate the sensing constraints of $\Robot$. \cref{alg:Vmorethan1 m=0-Patrolling Algorithm} works by partitioning $\Dcube$ into two disjoint sub-graphs, $\Dcube = \Dcube_1 \cup \Dcube_2$. $\Dcube_1$ contains all vertices in $\Dcube$ whose 1st coordinate is $1$, i.e., all vertices of the form $(1, *, *, \ldots *)$. $\Dcube_2$ contains all vertices whose 1st coordinate is greater than $1$, i.e., $\Dcube_2 = \Dcube \setminus \Dcube_1$. The behavior of \cref{alg:Vmorethan1 m=0-Patrolling Algorithm} is illustrated by \cref{fig:visualization.v>1.m=0}. As seen in the Figure, the Hamiltonian cycle constructed by \cref{alg:Vmorethan1 m=0-Patrolling Algorithm} is split into (i) a Hamiltonian path of $\Dcube_2$ and (ii) a Hamiltonian path of $\Dcube_1$. The Hamiltonian path of $\Dcube_2$ goes in a zig-zag, similar to \cref{alg:V=1 b=0-Patrolling Algorithm}, either increasing or decreasing $x_1$ until a coordinate of the form $(2,*, \ldots *)$ or $(n_1,*, \ldots *)$ is met (this is why \cref{alg:Vmorethan1 m=0-Patrolling Algorithm} requires $V > 1$  - otherwise we could not uniquely distinguish coordinates of the form $(2,*, \ldots *)$), taking a step along some non-$x_1$ coordinate, and then going in reverse along the $x_1$ axis. The conditions of \cref{theorem:grid graphsthatcanbepatrolledwithVvisibility0bits} imply $ \prod_{i=2}^{d} n_i $ is even when $d > 1$. This implies that the Hamiltonian path of $\Dcube_2$ shall end up adjacent to the start of $\Dcube_1$'s Hamiltonian path. Hence the two paths can be joined to form a Hamiltonian cycle of $\Dcube$.

\begin{figure}
    \centering
        \centering
        \resizebox{0.8\textwidth}{!}{
            \includegraphics[page=1]{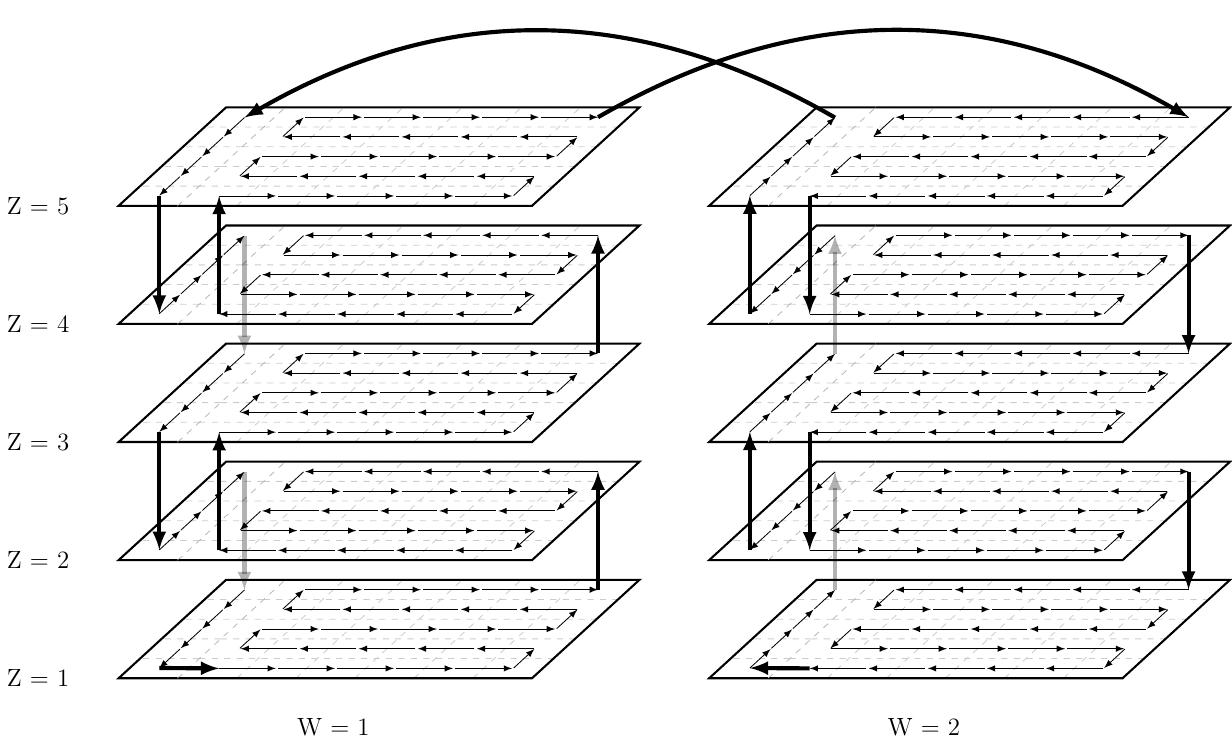}
        }
    \caption{The behavior of \cref{alg:Vmorethan1 m=0-Patrolling Algorithm} over $\Dcube = [7] \times [5] \times [5] \times [2]$.}
    \label{fig:visualization.v>1.m=0}
\end{figure}

\cref{alg:V=1 b=0-Patrolling Algorithm} and \cref{alg:Vmorethan1 m=0-Patrolling Algorithm}, together with Lemmas \ref{lemma:nopatroloddnumberofsensingregions} and \ref{lemma:nopatrol2x2sensingsubregion}, establish \cref{theorem:grid graphsthatcanbepatrolledwithVvisibility0bits}. \qed 

\section{Patrolling with \texorpdfstring{$1$}{1} Bit of Memory}
\label{section:patrol1bit}

\cref{section:patrol0bit} established that no $0$-bit algorithm can patrol all grid graphs no matter how large $\Robot$'s sensing range is. In contrast, in this section we give a $1$-bit, $V = 1$ algorithm that patrols all $d$-dimensional grid graphs. Additionally, unlike Algorithms \ref{alg:V=1 b=0-Patrolling Algorithm} and \ref{alg:Vmorethan1 m=0-Patrolling Algorithm}, we shall describe an algorithm that does not depend on the  dimension sizes $(n_1, \ldots n_d)$ of $\Dcube$, thus $\Robot$ does not need to know these in advance. 

Let us denote the current memory state of $\Robot$ ``$mem$''. Since we assume $\Robot$ possesses $1$ bit of memory, $mem$ is either $0$ or $1$.  The primary challenge we face  is that, with just $1$ bit of memory, each sensing region of $\Dcube$ can be exited in only $2$ directions - one direction when $mem = 0$ and one direction when $mem = 1$. However, patrolling a $d$-dimensional grid requires us to move in $2d$ distinct directions (as we need to be able to increase and decrease each of our $d$ coordinates). Consequently, we must somehow ``diffuse'' our directions of motion across the sensing regions of $\Dcube$ in a manner such that the in- and out- degree of each sensing region not exceed $2$, but $\Robot$ nevertheless manages to visit every vertex in $\Dcube$.

The number of sensing regions grows exponentially with $d$: it is between $2^d$ and $(2V+1)^d$ for a given $V$. Since $\Robot$'s algorithm can act differently in each sensing region, in some sense we have more ``leeway'' in higher dimensions. In our algorithm, we find that low-dimensional grid graphs require their own individual patrolling strategies, but that the patrolling strategy for higher dimensional graphs can be defined using induction, by calling these lower-dimensional patrolling algorithms as subroutines. Specifically,  our algorithm handles $1$-, $2$-, and $3$-dimensional grid graphs individually, but extends inductively to all grid graphs of dimension $4$ and above, using the algorithms for $2$- and $3$-dimensional graphs as building blocks.

\SetKwFunction{MakeMoveOneD}{MakeMove1D}
\SetKwFunction{MakeMoveTwoD}{MakeMove2D}
\SetKwFunction{MakeMoveThreeD}{MakeMove3D}
\SetKwFunction{MakeMovekplusoneD}{MakeMove(k+1)D}
\SetKwFunction{MakeMovekD}{MakeMove(k)D}
\SetKwFunction{MakeMoveiD}{MakeMove(i)D}
\SetKwFunction{MakeMovejD}{MakeMove(j)D}

\cref{subfig:visualization.m=1.1d} and \cref{subfig:visualization.m=1.2d} visually illustrate algorithms for patrolling a $1D$ and $2D$ grid graphs using $1$ bit of memory. We shall refer to these algorithms  as \MakeMoveOneD and \MakeMoveTwoD respectively. Pseudocode for  \MakeMoveOneD  and  \MakeMoveTwoD is available  - see \cref{alg:V=1 b=1-Patrolling Algorithm 1D} and \cref{alg:V=1 b=1-Patrolling Algorithm 2D}. We note that the pseudocode for these algorithms is written such that it can be run on any $d$-dimensional grid graph. This is deliberate, as we use lower-dimensional patrolling algorithms as sub-routines in our higher-dimensional patrolling algorithms. Running \MakeMoveOneD  and  \MakeMoveTwoD on higher-dimensional grid graphs will result in $\Robot$ patrolling a 1- or 2-dimensional sub-grid of $\Dcube$, respectively.

\begin{figure}[ht]
    \centering
    \begin{subfigure}{0.38\textwidth}
        \centering
        \resizebox{\textwidth}{!}{
            \includegraphics[page=1]{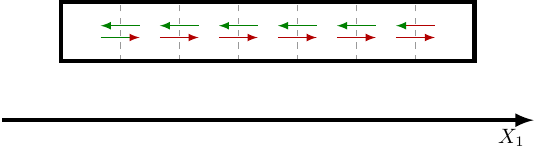}
        }
        \caption{}
        \label{subfig:visualization.m=1.1d}
    \end{subfigure}
    \hfill
    \begin{subfigure}{0.38\textwidth}
        \centering
        \resizebox{\textwidth}{!}{
            \includegraphics[page=1]{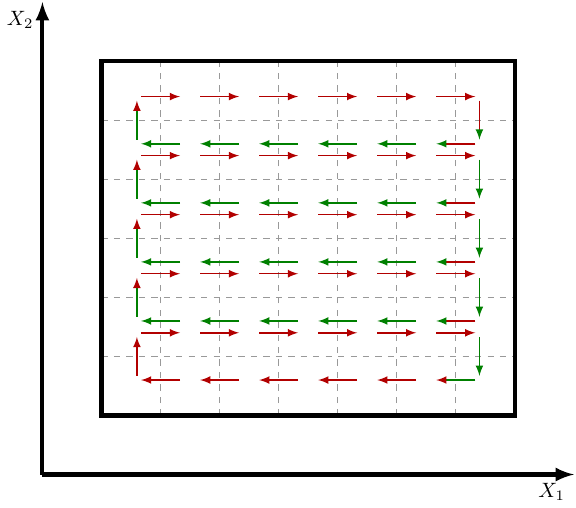}
        }
        \caption{}
        \label{subfig:visualization.m=1.2d}
    \end{subfigure}
    \caption{A visual illustration of patrolling 1D and 2D grid graphs with $1$ bit of memory using (a) \protect\MakeMoveOneD and (b) \protect\MakeMoveTwoD. Each arrow's body and head take on one of two colors: \textcolor{darkred}{red} or \textcolor{darkgreen}{green}. 
    When $\Robot$'s bit of memory is set to $0$, it searches for an arrow with a \textcolor{darkred}{red} body at its current location, and moves to where it points. Likewise, when $\Robot$'s bit of memory is set to $1$, it searches for an arrow with a \textcolor{darkgreen}{green} body, and moves to where it points. The color of the head of the arrow determines what $\Robot$ should set its bit of memory to after moving: $0$ if \textcolor{darkred}{red}, $1$ if  \textcolor{darkgreen}{green}. }
    \label{fig:visualization.m=1.1-2d}
\end{figure}

\begin{figure}[!ht]
    \centering
        \centering
        \resizebox{0.7\textwidth}{!}{
            \includegraphics[page=1]{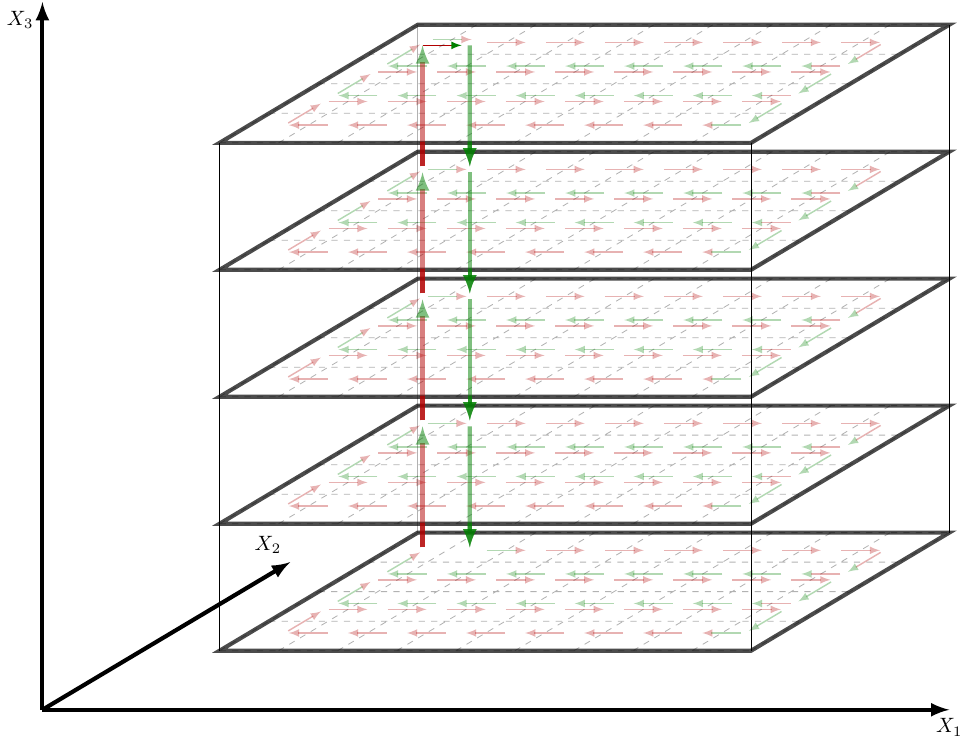}
        }
    \caption{Patrolling a 3D grid graph with $1$ bit of memory using \protect\MakeMoveThreeD. To be read the same way as \cref{fig:visualization.m=1.1-2d}. }
    \label{fig:visualization.m=1.3d}
\end{figure}

\begin{algorithm}[ht]
    
    \SetKwFunction{step}{step}
    \SetKwFunction{MakeMove}{MakeMove1D}

    \SetKw{And}{and}
    \SetKw{Not}{not}
    \SetKw{Or}{or}
    \SetKw{Continue}{continue}

    \SetArgSty{textrm}
    
    \BlankLine
    \SetKwProg{Fn}{Function}{}{}    
    \DontPrintSemicolon
    \Fn{\MakeMove}{
    \KwIn{$(\textit{l}_1, \textit{r}_1), mem \in \{0, 1\}$.}
        \BlankLine       
        $\step = (0, \ldots, 0)$ \tcp*{$d$ - dimensional zero vector}

        \BlankLine       
        \uIf{$mem=0$} {
        \BlankLine       
            \uIf{$r_1 \neq 0$}{
                $step[1] = 1$ \tcp*{move right}
            }
            \uElse{
                $step[1] = -1$ and $mem = 1$ \tcp*{change direction}
            }
        }
        \uElse{
            \uIf{$l_1 \neq 0$}{
                $step[1] = -1$ \tcp*{move left}
            }
            \uElse{
                $step[1] = 1$ and $mem = 0$ \tcp*{change direction}
            }
        }

        \BlankLine       
        \Return \step, $mem$
    }
\caption{A $1$-dimensional grid graph patrolling algorithm, using sensing range $V=1$  and $1$ bit of internal memory.}
\label{alg:V=1 b=1-Patrolling Algorithm 1D}	
\end{algorithm}

\begin{algorithm}[!bht]
    
    \SetKwFunction{step}{step}
    \SetKwFunction{MakeMove}{MakeMove2D}

    \SetKw{And}{and}
    \SetKw{Not}{not}
    \SetKw{Or}{or}
    \SetKw{Continue}{continue}

    \SetArgSty{textrm}
    
    \BlankLine
    \SetKwProg{Fn}{Function}{}{}    
    \DontPrintSemicolon
    \Fn{\MakeMove}{
    \KwIn{$\{(\textit{l}_1, \textit{r}_1), (\textit{l}_2, \textit{r}_2)\}$, $mem \in \{0, 1\}$.}
        \BlankLine       
        $\step = (0, \ldots, 0)$ \tcp*{$d$ - dimensional zero vector}
        \BlankLine
        \uIf{$mem=0$} {
            \lIf{$r_1 \neq 0\;\And\;l_2 \neq 0$} {
                $\step[1] = 1$
            }
            \lIf{$l_1 \neq 0\;\And\;l_2 = 0$} {
                $\step[1] = -1$
            }
            \lIf{$l_1 = 0\;\And\;l_2 = 0$} {
                $\step[2] = 1$
            }
            \lIf{$r_1 = 0\;\And\;l_2 \cdot r_2 \neq 0$} {
                $\step[1] = -1$ and 
                $mem = 1$
            }
            \lIf{$r_1 = 0\;\And\;r_2 = 0$} {
                $\step[2] = -1$ and
                $mem = 1$
            }
        }
        \BlankLine
        \uElse{
            \lIf{$l_1 = 0\;\And\;r_2 \neq 0$} {
                $\step[2] = 1$ and
                $mem = 0$
            }
            \lIf{$l_1 = 0\;\And\;r_2 = 0$} {
                $\step[1] = 1$ and
                $mem = 0$
            }
            \lIf{$l_1 \cdot r_1 \neq 0\;\And\;r_2 \neq 0$} {
                $\step[1] = -1$
            }
            \lIf{$l_1 \cdot r_1 \neq 0\;\And\;r_2 = 0$} {
                $\step[1] = 1$
            }
            \lIf{$r_1 = 0\;\And\;l_2 = 0$} {
                $\step[1] = -1$ and
                $mem = 0$
            }
            \lIf{$r_1 = 0\;\And\;l_2 \neq 0$} {
                $\step[2] = -1$
            }                
        }

        \BlankLine
        \Return \step, $mem$\;
    }

\caption{A $2$-dimensional grid graph patrolling algorithm, using sensing range $V=1$  and $1$ bit of internal memory.}
\label{alg:V=1 b=1-Patrolling Algorithm 2D}	
\end{algorithm}

We note that the visual descriptions of \MakeMoveOneD and \MakeMoveTwoD given in \cref{subfig:visualization.m=1.2d} are not fully determined: some vertices have only one outgoing arrow. For example, in \cref{subfig:visualization.m=1.2d}, if $\Robot$ is located at $(1,1)$ with $mem = 1$, its next move is undefined. Such under-determined states can easily be handled - and we handle them explicitly in  the pseudocode of \MakeMoveTwoD given in \cref{alg:V=1 b=1-Patrolling Algorithm 2D} -  but we deliberately do not depict them in \cref{fig:visualization.m=1.1-2d} because they are ``transient'': $\Robot$ never returns to these states after exiting them.  

Because $\Robot$ has just $1$ bit of memory, there can be at most two outgoing arrows per vertex $v \in \Dcube$ in \cref{subfig:visualization.m=1.2d}, thus transient states give us important ``degrees of freedom'' when we want to extend our patrolling algorithms to higher dimensions. For example, in our algorithm for patrolling 3-dimensional grid graphs, \MakeMoveThreeD, we turn some transient states into non-transient states that send $\Robot$ up and down the $x_3$ axis. To explain how \MakeMoveThreeD works we make use of ``floors'':

\begin{definition}
\label{definition:floorceiling}

Let $\Dcube$ be a $d$-dimensional grid graph. A \textbf{$k$-floor of $\Dcube$} with coordinates $(q_{k+1}, q_{k+2}, \ldots, q_d)$ is the sub-graph of $\Dcube$ defined by:

\begin{equation*}
F_k(q_{k+1}, q_{k+2}, \ldots, q_d) = \Set{(x_1,\ldots,x_k,q_{k+1},\ldots, q_d) \in \Dcube}{(x_1,\ldots,x_k) \in [n_1] \times [n_2] \times \ldots [n_k]}
\end{equation*}



\end{definition}

We introduce the notation $(q_1,q_2,\ldots, q_d|m)$ to refer to the state of $\Robot$ where $x_i = q_i$ and $mem = m$. We shall also write $(q_1,\ldots,q_k,*|m)$ to refer to the set of states where $x_1 = q_1, \ldots x_k = q_k$, but $x_{k+1}, \ldots, x_{d}$ can take on any value. Finally, we use the notation $\neg q$ to say ``any value except q''. For example, $(\neg q_1, q_2, \ldots, q_d|m)$ refers to a set of states where $mem = m$, $x_1 \neq q_1$, and $x_2 = q_2, \ldots, x_d = q_d$.

\MakeMoveThreeD works by moving $\Robot$ according to \MakeMoveTwoD in every $2$-floor of $\Dcube$, but it moves $\Robot$ up and down $x_3$ whenever $\Robot$ reaches the states $(1, n_2, \neg n_3,*|0)$ and  $(\neg 1,n_2,\neg 1,*|1)$, respectively. It also reserves the state $(1,n_2,n_3,*|0)$ for transitioning between the latter two states. Pseudocode and diagrams for \MakeMoveThreeD are available  - see \cref{alg:V=1 b=1-Patrolling Algorithm 3D} and \cref{fig:visualization.m=1.3d}. 

The idea behind \MakeMoveThreeD seems natural, as it simply runs \MakeMoveTwoD until it finishes patrolling a $2$-floor and then goes to a different $2$-floor. However, because we have just $1$ bit of memory, there is more nuance to this than might be suspected at a glance: the states $(1, n_2, \neg n_3,*|0)$ and  $(\neg 1,n_2,\neg 1,*|1)$ were specifically chosen because $\Robot$ can transition between them in a single move, and because $(\neg 1,n_2|1)$ is a transient state of \MakeMoveTwoD that can be repurposed to enable $\Robot$ to move in the $x_3$ axis. Likewise, to use \MakeMoveThreeD as a subroutine in our higher-dimensional patrolling algorithm, \MakeMovekD, we must find other transient states that do not interfere with the ones used by \MakeMoveThreeD. Hence, there is \textit{a small economy of transient states} that must be exploited; it took the authors some time to find a construction that works.

\begin{algorithm}[ht]
    
    \SetKwFunction{step}{step}
    \SetKwFunction{MakeMove}{MakeMove3D}
    \SetKwFunction{MakeMoveTwoD}{MakeMove2D}

    \SetKw{And}{and}
    \SetKw{Not}{not}
    \SetKw{Or}{or}
    \SetKw{Continue}{continue}

    \SetArgSty{textrm}
    
    \BlankLine
    \SetKwProg{Fn}{Function}{}{}    
    \DontPrintSemicolon
    \Fn{\MakeMove}{
    \KwIn{$\{(\textit{l}_1, \textit{r}_1), (\textit{l}_2, \textit{r}_2), (\textit{l}_3, \textit{r}_3)\}$, $mem \in \{0, 1\}$.}
        \BlankLine       
        $\step = (0, \ldots, 0)$ \tcp*{$d$ - dimensional zero vector}

        \BlankLine       
        \uIf{$l_1=0\;\And\;r_2=0\;\And\;mem=0$} {
        \BlankLine       
            \uIf{$r_3 = 0$}{
                $\step[1] = 1$ \tcp*{special state that occurs when $x_3$ is maximized}
                \label{Algo3D:2nd-case}
            }
            \uElse{
                $\step[3] = 1$ \tcp*{move up $x_3$}
                \label{Algo3D:1st-case}
            }
            $mem = 1$
            
            \Return \step, $mem$
        }

        \BlankLine       
        \uIf{$l_1, l_3 \neq 0\;\And\;r_2=0\;\And\;mem=1$} {
            $step[3] = -1$ \tcp*{move down $x_3$}
            \label{Algo3D:3rd-case}
            \Return \step, $mem$
        }

        \BlankLine       
        \Return \MakeMoveTwoD($\{(\textit{l}_1, \textit{r}_1), (\textit{l}_2, \textit{r}_2)\}, \textit{mem}$)
    }
\caption{A $3$-dimensional grid graph patrolling algorithm, using sensing range $V=1$  and $1$ bit of internal memory.}
\label{alg:V=1 b=1-Patrolling Algorithm 3D}	
\end{algorithm}

The correctness of \MakeMoveOneD, \MakeMoveTwoD and \MakeMoveThreeD is asserted by \cref{lemma:makemoveonetwothreeD_succeed}:

\begin{lemma}
\label{lemma:makemoveonetwothreeD_succeed}

Let $\Dcube$ be a $d$-dimensional grid graph and let $v$ be $\Robot$'s current location. 

\begin{enumerate}
    \item \MakeMoveOneD patrols the $1$-floor of $\Dcube$ that contains $v$.
    \item \MakeMoveTwoD patrols the $2$-floor of $\Dcube$ that contains $v$.
    \item \MakeMoveThreeD patrols the $3$-floor of $\Dcube$ that contains $v$.
\end{enumerate}

\end{lemma}

Note that since any $d$-dimensional grid graph is a $d$-floor of itself, \cref{lemma:makemoveonetwothreeD_succeed} implies that  \MakeMoveOneD, \MakeMoveTwoD and \MakeMoveThreeD successfully patrol $1$-, $2$- and $3$-dimensional grid graphs, respectively. 

\begin{proof}
\textbf{(1)} can readily be inferred by looking at \cref{subfig:visualization.m=1.1d} - on any given $1$-floor of $\Dcube$, \MakeMoveOneD will simply increase and decrease $x_1$ so as to move between $x_1 = 1$ and $x_1 = n_1$ indefinitely,  patrolling the floor.

\textbf{(2)} can also readily be inferred by looking at \cref{subfig:visualization.m=1.2d}. As can be seen in \cref{subfig:visualization.m=1.2d}, \MakeMoveTwoD works by having $\Robot$ explore $1$-floors in which $1 < x_2 < n_2$ in a similar fashion as \MakeMoveOneD. When $\Robot$ finishes covering such a $1$-floor and has returned to a vertex in which $x_1 = 1$, it moves to the next $1$-floor by incrementing $x_2$. When it finally enters a $1$-floor in which $x_2 = n_2$, it must do so via a vertex of the form $(1,n_2,*)$ (it is possible for $\Robot$ to be \textit{initiated} in such a floor, but not enter it, at a different location). When this happens, \MakeMoveTwoD increments $x_1$ to $n_1$ (thus visiting all vertices of the $1$-floor),  decreases $x_2$ down to $1$, and finally decreases $x_1$ back to $1$, restarting the patrol path. This results in $\Robot$ visiting all vertices of the $2$-floor (note that if $\Robot$ is not initiated with $x_1 = x_2 = 1$, it might need to arrive at $x_1 = x_2 = 1$ twice before it has visited all vertices of the $2$-floor). 

We now prove \textbf{(3)}. The following facts follow from the  pseudocode of \MakeMoveThreeD (\cref{alg:V=1 b=1-Patrolling Algorithm 3D}):

\begin{enumerate}[label=(\alph*)]
\item Outside of the three states  $(1, n_2, \neg n_3,*|0)$, $(\neg 1, n_2, \neg 1,*|1)$, and $(1,n_2,n_3,*|0)$, \MakeMoveThreeD simply executes \MakeMoveTwoD. 
\item \MakeMoveThreeD can only enter a $2$-floor with higher $x_3$-coordinate than its current $2$-floor's in the state  $(1,n_2,*|1)$.
\item \MakeMoveThreeD must transition to the state $(1,n_2,*|0)$ to increment $x_3$ (note the difference from the state mentioned in (b) - this state has $mem = 0$).
\item Assuming $\Robot$ is currently in the state $(1,n_2,*|1)$, it can only enter the state $(1,n_2,*|0)$ after \MakeMoveThreeD has visited all vertices in the current $2$-floor (because \MakeMoveThreeD runs \MakeMoveTwoD until this state occurs, and \MakeMoveTwoD visits all these vertices). 
\item Once \MakeMoveThreeD finishes visiting all vertices of the $x_3 = n_3$  $2$-floor it arrives at the state $(1, n_2, n_3,*|0)$, subsequently transitions to the state $(\neg 1, n_2, n_3,*|1)$ by incrementing $x_1$, and then decrements $x_3$ all the way back to $1$.
\item Upon returning to the $x_3 = 1$ $2$-floor via the chain of events described in (e), \MakeMoveThreeD visits all vertices in that floor before reaching the state $(1,n_2,1,*|0)$ and beginning the patrol path again.
\end{enumerate}

By (a)-(f), \MakeMoveThreeD indefinitely cycles between all $2$-floors of the current $3$-floor by incrementing $x_3$ to $n_3$ and then decrementing $x_3$ back to $1$. Specifically by (c), (d) and (f), every time \MakeMoveThreeD enters a $2$-floor by incrementing $x_3$, and also when it enters the $2$-floor with $x_3 = 1$ by decrementing $x_3$, it visits all vertices of that floor. This implies (3).

(As visual aid for the above analysis, we refer the reader to \cref{fig:visualization.m=1.3d}. In the Figure, the \textbf{bold}-colored arrows show how \MakeMoveThreeD cycles between $2$-floors. Besides the bold arrows, its action at each vertex is based on \MakeMoveTwoD.)
\end{proof}

In general, our Algorithm for patrolling $k+1$-dimensional grid graphs shall be called \MakeMovekplusoneD. This algorithm is built recursively, calling \MakeMovekD as a subroutine. The concept of \MakeMovekplusoneD is similar to \MakeMoveThreeD: it has a small number of \textit{specially-designated states} whose purpose is to increment and decrement $x_{k+1}$, and at all other times it executes \MakeMovekD. Pseudocode for \MakeMovekplusoneD is available  - see \cref{alg:V=1 b=1-Patrolling Algorithm (k+1)D}. The algorithm can also be defined non-inductively, i.e., without calling \MakeMovekD: see \cref{appendix:nonrecursivepseudocode}.

\begin{algorithm}[ht]
    
    \SetKwFunction{step}{step}
    \SetKwFunction{MakeMove}{MakeMove(k+1)D}
    \SetKwFunction{MakeMovekD}{MakeMove(k)D}

    \SetKw{And}{and}
    \SetKw{Not}{not}
    \SetKw{Or}{or}
    \SetKw{Continue}{continue}

    \SetArgSty{textrm}
    
    \BlankLine
    \SetKwProg{Fn}{Function}{}{}    
    \DontPrintSemicolon
    \Fn{\MakeMove}{
    \KwIn{$\{(\textit{l}_1, \textit{r}_1), \ldots, (\textit{l}_{k+1}, \textit{r}_{k+1})\}$, $mem \in \{0, 1\}$.}
        \BlankLine       
        $\step = (0, \ldots, 0)$ \tcp*{$d$ - dimensional zero vector}

        \BlankLine       
        \uIf{$\bigwedge\limits_{j=2}^{k-1}l_j=0\;\And\;r_{k}=0\;\And\;mem=1$} {
        \BlankLine       
            \uIf{$r_1 = 0$}{
                \uIf{$r_{k+1} = 0$} {
                    $step[1] = -1$
                    \label{AlgokD:1st-case}
                }
                \uElse{
                    $\step[k+1] = 1$ \tcp*{move up $x_{k+1}$}
                    \label{AlgokD:2nd-case}
                    $mem = 0$
                }
                \Return \step, $mem$
            }
            \uIf{$l_{k+1} \neq 0$} {
                $\step[k+1] = -1$ \tcp*{move down $x_{k+1}$}
                \label{AlgokD:3rd-case}
                \Return \step, $mem$
            }
        }

        \BlankLine       
        \Return \MakeMovekD($\{(\textit{l}_1, \textit{r}_1), \ldots, (\textit{l}_k, \textit{r}_k)\}, \textit{mem}$)
    }
\caption{A $(k+1)$-dimensional grid graph patrolling algorithm, using sensing range $V=1$  and $1$ bit of internal memory.}
\label{alg:V=1 b=1-Patrolling Algorithm (k+1)D}	
\end{algorithm}

The main difference between \MakeMovekD, $k > 3$ and \MakeMoveThreeD is that the specially-designated states of \MakeMovekD are $(n_1,1, \ldots, 1, n_{k-1}, n_{k} ,*|1)$, $(n_1,1, \ldots, 1, n_{k-1}, \neg n_{k} ,*|1)$ and $(\neg n_1,1, \ldots, 1, n_{k-1},\neg 1,*|1)$, whereas $\MakeMoveThreeD$ specially designates the states $(1,n_2,n_3,*|0)$, $(1,n_2,\neg n_3,*|0)$ and $(\neg 1,n_2,\neg 1,*|1)$. For some intuition about why we choose different types of states in the $3$-dimensional case, first note that at the lowest level of recursion, \MakeMovekplusoneD executes \MakeMoveTwoD - and it only deviates from \MakeMoveTwoD in a small number of states. Thus, \MakeMovekplusoneD can be understood as executing \MakeMoveTwoD until, in a given  $2$-floor, it hits a specially-designated state of some higher-dimensional \MakeMoveAll algorithm. Consequently, it is very important that these designated states do not break \MakeMoveTwoD's ability to patrol $2$-floors. The $x_1$ and $x_2$ coordinates of the specially-designated states of \MakeMovekD, $k > 3$, are always $(n_1, 1)$ and $(\neg n_1, 1)$. However, if instead of calling \MakeMoveThreeD as a subroutine of $\MakeMovekD$ when $k = 4$, we naively tried to call \MakeMovekD with $k = 3$, we see from the pseudocode that its designated states'  $x_1$ and $x_2$ coordinates would have been  $(n_1,n_2)$ and $(\neg n_1, n_2)$ - requiring $2$-floors to specially handle these coordinates in addition to $(n_1, 1)$ and $(\neg n_1, 1)$. It is difficult to repurpose these coordinates without breaking \MakeMoveTwoD, and we found it easier, instead, to have \MakeMoveThreeD  interact with the $(1,n_2)$ and $(\neg 1, n_2)$ states of $2$-floors.


\begin{lemma}
Let $\Dcube$ be a $d$-dimensional grid graph and let $v$ be $\Robot$'s current location.  \MakeMovekD patrols the $k$-floor of $\Dcube$ that contains $v$.
\label{lemma:makemovekworks}
\end{lemma}

\begin{proof}

We prove the Lemma by induction. The base cases $k = 1..3$ are given by   \cref{lemma:makemoveonetwothreeD_succeed}. Let us assume the Lemma's statement is true for \MakeMovekD, $k \geq 3$, and prove that it remains true for \MakeMovekplusoneD. The proof idea will be similar to our analysis of \MakeMoveThreeD in \cref{lemma:makemoveonetwothreeD_succeed}.

\begin{table}[!ht]
\centering
\caption{Overview of \protect\MakeMovekplusoneD (\cref{alg:V=1 b=1-Patrolling Algorithm (k+1)D}) based on $\Robot$'s current coordinate and $mem$.  ``Lines'' refers to the relevant lines in the pseudocode.}
\begin{tabular}{|c|c|c|c|}
\hline
\textbf{State} & \textbf{Lines} & \textbf{Action} \\
\hline\hline
$(1,n_2,n_3,*|0)$ & \cref{alg:V=1 b=1-Patrolling Algorithm 3D}, \cref{Algo3D:2nd-case} & Increase $x_1$ and set $mem=1$ \\ \hline
$(1,n_2,\neg n_3,*|0)$ & \cref{alg:V=1 b=1-Patrolling Algorithm 3D}, \cref{Algo3D:1st-case} & Increase $x_3$ and set $mem=1$\\ \hline
$(\neg 1,n_2,\neg 1,*|1)$ & \cref{alg:V=1 b=1-Patrolling Algorithm 3D}, \cref{Algo3D:3rd-case}& Decrease $x_3$ \\ \hline
$(n_1,1, \ldots, 1, n_k, n_{k+1} ,*|1)$, $k < d$  & \cref{alg:V=1 b=1-Patrolling Algorithm (k+1)D}, \cref{AlgokD:1st-case} & Decrease $x_{1}$ \\ \hline
$(n_1,1, \ldots, 1, n_k, \neg n_{k+1} ,*|1)$, $k < d$  & \cref{alg:V=1 b=1-Patrolling Algorithm (k+1)D}, \cref{AlgokD:2nd-case} & Increase $x_{k+1}$ and set $mem=0$ \\ \hline
$(\neg n_1,1, \ldots, 1, n_k,\neg 1,*|1)$, $k < d$ & \cref{alg:V=1 b=1-Patrolling Algorithm (k+1)D}, \cref{AlgokD:3rd-case} & Decrease $x_{k+1}$ \\ \hline
Else & \  -  & Run \MakeMoveTwoD (\cref{alg:V=1 b=1-Patrolling Algorithm 2D}) \\ \hline
\end{tabular}
\label{table:actionsofmakemovekplus1d}
\end{table}

The actions of \MakeMovekplusoneD in any given state are written out in \cref{table:actionsofmakemovekplus1d} (as can be confirmed from the pseudocode). We see from \cref{table:actionsofmakemovekplus1d} that $\Robot$ always enters a new $k$-floor with higher $x_k$ coordinate in the state $(n_1,1, \ldots, 1, n_k ,*|0)$. In fact, we readily verify from the Table that $\Robot$ \textit{only} enters the state $(n_1,1, \ldots, 1, n_k ,*|0)$ after it increments $x_{k+1}$ via the state $(n_1,1, \ldots, 1, n_k , \neg n_{k+1}, *|1)$ (no other row of \cref{table:actionsofmakemovekplus1d} transitions to this state; in particular,  \MakeMoveTwoD (\cref{alg:V=1 b=1-Patrolling Algorithm 2D}) always enters the state $(n_1,1,*)$ with $mem = 1$, thus cannot transition to the state $(n_1,1, \ldots, 1, n_k ,*|0)$). 

Let $(q_1, q_2, \ldots q_d)$ be $\Robot$'s current location. By the inductive assumption, \MakeMovekD patrols the $F_k(q_{k+1},\ldots,q_d)$ floor. We infer from the previous paragraph that \MakeMovekD must visit all vertices of $F_k(q_{k+1},\ldots,q_n)$ between each visit to $(n_1,1, \ldots, 1, n_k ,q_{k+1}, \ldots q_d)$, since as long as $\Robot$ is in the $F_k(q_{k+1},\ldots,q_d)$ floor, \MakeMovekD can only enter the $(n_1,1, \ldots, 1, n_k ,q_{k+1}, \ldots q_d)$  vertex with $mem = 1$ (i.e., it can only enter this vertex with $mem = 0$ \textit{from another floor}), hence $\Robot$ must walk the exact same path whenever it reaches that location.

\MakeMovekplusoneD executes \MakeMovekD until arriving at one of the states $(n_1,1, \ldots, 1, n_k, \neg n_{k+1} ,*|1)$, $(\neg n_1,1, \ldots, 1, n_k,\neg 1,*|1)$, or $(n_1,1, \ldots, 1, n_k, n_{k+1} ,*|1)$. $\Robot$ can only transition to the state $(\neg n_1,1, \ldots, 1, n_k,\neg 1,*|1)$ from this state itself or from $(n_1,1, \ldots, 1, n_k, \neg n_{k+1} ,*|1)$. Because of this, and because whenever \MakeMovekplusoneD enters a new floor by incrementing $x_{k+1}$ it does so in the state $(n_1,1, \ldots, 1, n_k, \neg n_{k+1} ,*|0)$, we see from the prior paragraph that whenever \MakeMovekplusoneD enters a new $k$-floor by incrementing $x_{k+1}$, it runs \MakeMovekD until it completely patrols that $k$-floor. In addition, $\Robot$ finishes patrolling the $k$-floor it entered by incrementing $x_{k+1}$ \textit{precisely} when it reaches the state $(n_1,1, \ldots, 1, n_k, *|1)$, at which point it either increments $x_{k+1}$ (if $x_{k+1} \neq n_{k+1}$), or it transitions to $(\neg n_1,1, \ldots, 1, n_k,\neg 1,*|1)$ by decrementing $x_1$ from $n_1$ to $n_1 - 1$, and subsequently decrements $x_k$ all the way down to $1$. 

To summarize what we have so far: we know that \MakeMovekplusoneD repeatedly cycles through all $k$-floors of the current $k+1$-floor, and that it visits all vertices of any $k$-floor it enters by incrementing $x_{k+1}$. Thus, \MakeMovekplusoneD eventually visits all vertices of $k$-floors of the current $k+1$-floor whose $x_{k+1}$ coordinate is greater than $1$. 

It remains to show that the (unique) $k$-floor of the current $k+1$-floor whose $x_{k+1}$ coordinate is $1$ is also fully visited. In the steady state of cycling through $k$-floors described above, this $k$-floor is necessarily entered via the state $(\neg n_1,1, \ldots, 1, n_k,1,*|1)$, upon decrementing $x_{k+1}$. More specifically, $\Robot$'s $x_1$ coordinate when entering this $k$-floor must be $n_1 - 1$, so its state is $(n_1 - 1,1, \ldots, 1, n_k,1,*|1)$. From the state $(n_1 - 1,1, \ldots, 1, n_k,1,*|1)$, \MakeMovekplusoneD runs \MakeMoveTwoD, which decrements $x_1$ until $\Robot$ arrives at the state $(1,1, \ldots, 1, n_k,1,*|1)$, after which $\Robot$ transitions to $(1,2, \ldots, 1, n_k,1,*|0)$. When $\Robot$ enters a $k$-floor in the state $(n_1,1, \ldots, 1, n_k, \neg n_{k+1} ,*|0)$, it decrements $n_1$ until arriving at $(1,1, \ldots, 1, n_k, \neg n_{k+1} ,*|0)$, and subsequently transitions to $(1,2, \ldots, 1, n_k, \neg n_{k+1} ,*|0)$ as well. Thus in both cases $\Robot$ visits the same vertices in the same order (albeit with different $mem$ values) and arrives at the state $(1,2, \ldots, 1, n_k, \neg n_{k+1} ,*|0)$). We have already proven earlier that when $\Robot$ enters the state $(n_1,1, \ldots, 1, n_k, \neg n_{k+1} ,*|0)$, it visits the entire $k$-floor before arriving at the state $(n_1,1, \ldots, 1, n_k, \neg n_{k+1} ,*|1)$. We infer from this that $\Robot$ also visits the entire  $x_{k+1} = 1$ $k$-floor when entering it in the state $(n_1-1,1, \ldots, 1, n_k,1,*|1)$. This completes the proof.
\end{proof}

Let us give the name ``\MakeMoveAll'' to the algorithm that patrols a $d$-dimensional grid graph $\Dcube$ by running $\MakeMovekD$ with $k = d$  (where $d$ is inferred from $\Robot$'s sensing data, $\Sense{1}{\Position}$). We can now prove the main theorem of this section:

\begin{theorem}
Let $\Dcube$ be any grid graph. \MakeMoveAll patrols $\Dcube$, visiting every vertex in at most $2|\Dcube|$ steps, using $V=1$ sensing range and $1$ bit of memory.
\label{theorem:canpatrolwith1bit}
\end{theorem}

\begin{proof}
\cref{lemma:makemovekworks} implies that \MakeMoveAll patrols $\Dcube$. As can be seen from the pseudocode, \MakeMovekD uses exactly $1$ bit of memory and sensing range $V=1$, and consequently so does \MakeMoveAll. Since \MakeMoveAll uses $1$ bit of memory, it can visit any vertex at most twice before resetting its patrol path, and so \MakeMoveAll takes at most $2|\Dcube|$ steps to visit each vertex. 
\end{proof}

It may be of note that, due to details of the recursion, despite \cref{lemma:makemovekworks}'s proof possibly (mis)leading one to think so, \MakeMoveAll does not traverse typically traverse entire floors at once: especially when $k \leq d - 2$, \MakeMoveAll will often depart from $k$-floors without completing them, only to return later.

As mentioned earlier, \MakeMoveAll can also be written explicitly, without recursive calls to $\MakeMovekD$  - see \cref{alg:V=1 b=1-Patrolling Algorithm 3+D}.

\section{Discussion}

We set out to answer whether a given $d$-dimensional grid graph can be patrolled with sensing range $V$ and $b$ bits of memory. Our results settle this question: \cref{theorem:canpatrolwith1bit} shows that $b = 1$ and $V = 1$ are sufficient to patrol any grid graph, and \cref{theorem:grid graphsthatcanbepatrolledwithVvisibility0bits} characterizes the values of $V$ for which a grid graph can be patrolled with $b = 0$ bits of memory. Our $1$-bit patrolling result is surprising because it defies the intuition, based on existing complexity bounds, that memory requirements should grow with the graph's maximum degree.  This demonstrates that highly regular environments like grids admit extremely space-efficient patrolling algorithms compared to arbitrary graphs. 

An important direction for future work is to establish the space complexity of patrolling broader classes of highly structured environments beyond grids, such as other subgraph families of $\mathbb{Z}^d$. As explained in the introduction,  our $1$-bit algorithm, and the techniques we developed here, such as analyzing sensing regions,  provide a foundation to potentially obtain improved space complexity results for more general environments, such as Cartesian products of general $Z^d$ subspaces. More broadly, mapping the space complexity landscape of patrolling in structured environments is an important challenge where progress seems possible by exploiting inherent regularities in the environments under consideration.

Overall, our work gives concrete evidence that very little memory is needed for patrolling highly regular environments compared to arbitrary graphs. Deepening our understanding of the computational resources required for patrolling in different environments has important implications for diverse real-world domains, from operations research to robotics, where patrolling is a fundamental problem.

\bibliography{bibliography} 

\begin{thebibliography}{10}

\bibitem{agmon2009multi}
Noa Agmon.
\newblock {\em Multi-robot patrolling and other multi-robot cooperative tasks: An algorithmic approach}.
\newblock Bar Ilan University, 2009.

\bibitem{alon1990universal}
Noga Alon, Yossi Azar, and Yiftach Ravid.
\newblock Universal sequences for complete graphs.
\newblock {\em Discrete Applied Mathematics}, 27(1-2):25--28, 1990.

\bibitem{amir_minimizing_2019}
Michael Amir and Alfred~M. Bruckstein.
\newblock Minimizing {Travel} in the {Uniform} {Dispersal} {Problem} for {Robotic} {Sensors}.
\newblock In {\em Proceedings of the 18th {International} {Conference} on {Autonomous} {Agents} and {MultiAgent} {Systems}}, {AAMAS} '19, pages 113--121, Richland, SC, 2019. International Foundation for Autonomous Agents and Multiagent Systems.

\bibitem{arkin2000approximation}
Esther~M Arkin, S{\'a}ndor~P Fekete, and Joseph~SB Mitchell.
\newblock Approximation algorithms for lawn mowing and milling.
\newblock {\em Computational Geometry}, 17(1-2):25--50, 2000.

\bibitem{bampas2010almost}
Evangelos Bampas, Jurek Czyzowicz, Leszek G{\k{a}}sieniec, David Ilcinkas, and Arnaud Labourel.
\newblock Almost optimal asynchronous rendezvous in infinite multidimensional grids.
\newblock In {\em Distributed Computing: 24th International Symposium, DISC 2010, Cambridge, MA, USA, September 13-15, 2010. Proceedings 24}, pages 297--311. Springer, 2010.

\bibitem{beame1998time}
Paul Beame, Allan Borodin, Prabhakar Raghavan, Walter~L Ruzzo, and Martin Tompa.
\newblock A time-space tradeoff for undirected graph traversal by walking automata.
\newblock {\em SIAM Journal on Computing}, 28(3):1051--1072, 1998.

\bibitem{bender1998power}
Michael~A Bender, Antonio Fern{\'a}ndez, Dana Ron, Amit Sahai, and Salil Vadhan.
\newblock The power of a pebble: Exploring and mapping directed graphs.
\newblock In {\em Proceedings of the thirtieth annual ACM symposium on Theory of computing}, pages 269--278, 1998.

\bibitem{budach1978automata}
Lothar Budach.
\newblock Automata and labyrinths.
\newblock {\em Mathematische Nachrichten}, 86(1):195--282, 1978.

\bibitem{bui2023optimal}
Thuy Bui and Thomas Lidbetter.
\newblock Optimal patrolling strategies for trees and complete networks.
\newblock {\em European Journal of Operational Research}, 311(2):769--776, 2023.

\bibitem{icalpchalopin2010rendezvous}
J{\'e}r{\'e}mie Chalopin and Shantanu Das.
\newblock Rendezvous of mobile agents without agreement on local orientation.
\newblock In {\em Automata, Languages and Programming: 37th International Colloquium, ICALP 2010, Bordeaux, France, July 6-10, 2010, Proceedings, Part II 37}, pages 515--526. Springer, 2010.

\bibitem{cohen2017exploring}
Lihi Cohen, Yuval Emek, Oren Louidor, and Jara Uitto.
\newblock Exploring an infinite space with finite memory scouts.
\newblock In {\em Proceedings of the Twenty-Eighth Annual ACM-SIAM Symposium on Discrete Algorithms}, pages 207--224. SIAM, 2017.

\bibitem{cohen2008label}
Reuven Cohen, Pierre Fraigniaud, David Ilcinkas, Amos Korman, and David Peleg.
\newblock Label-guided graph exploration by a finite automaton.
\newblock {\em ACM Transactions on Algorithms (TALG)}, 4(4):1--18, 2008.

\bibitem{czyzowicz2012more}
Jurek Czyzowicz, Stefan Dobrev, Leszek G{\k{a}}sieniec, David Ilcinkas, Jesper Jansson, Ralf Klasing, Ioannis Lignos, Russell Martin, Kunihiko Sadakane, and Wing-Kin Sung.
\newblock More efficient periodic traversal in anonymous undirected graphs.
\newblock {\em Theoretical Computer Science}, 444:60--76, 2012.

\bibitem{diks2004tree}
Krzysztof Diks, Pierre Fraigniaud, Evangelos Kranakis, and Andrzej Pelc.
\newblock Tree exploration with little memory.
\newblock {\em Journal of Algorithms}, 51(1):38--63, 2004.

\bibitem{disser2016undirected}
Yann Disser, Jan Hackfeld, and Max Klimm.
\newblock Undirected graph exploration with $\theta(\log \log n)$ pebbles.
\newblock In {\em Proceedings of the Twenty-Seventh Annual ACM-SIAM Symposium on Discrete Algorithms}, pages 25--39. SIAM, 2016.

\bibitem{dobrev2005finding}
Stefan Dobrev, Jesper Jansson, Kunihiko Sadakane, and Wing-Kin Sung.
\newblock Finding short right-hand-on-the-wall walks in graphs.
\newblock In {\em Structural Information and Communication Complexity: 12 International Colloquium, SIROCCO 2005, Mont Saint-Michel, France, May 24-26, 2005. Proceedings 12}, pages 127--139. Springer, 2005.

\bibitem{dobrev2019exploration}
Stefan Dobrev, Lata Narayanan, Jaroslav Opatrny, and Denis Pankratov.
\newblock Exploration of high-dimensional grids by finite automata.
\newblock In {\em 46th International Colloquium on Automata, Languages, and Programming (ICALP 2019)}. Schloss Dagstuhl-Leibniz-Zentrum fuer Informatik, 2019.

\bibitem{fraigniaud2005graph}
Pierre Fraigniaud, David Ilcinkas, Guy Peer, Andrzej Pelc, and David Peleg.
\newblock Graph exploration by a finite automaton.
\newblock {\em Theoretical Computer Science}, 345(2-3):331--344, 2005.

\bibitem{fraigniaud2005space}
Pierre Fraigniaud, David Ilcinkas, Sergio Rajsbaum, and S{\'e}bastien Tixeuil.
\newblock Space lower bounds for graph exploration via reduced automata.
\newblock In {\em Structural Information and Communication Complexity: 12 International Colloquium, SIROCCO 2005, Mont Saint-Michel, France, May 24-26, 2005. Proceedings 12}, pages 140--154. Springer, 2005.

\bibitem{gkasieniec2008fast}
Leszek G{\k{a}}sieniec, Ralf Klasing, Russell Martin, Alfredo Navarra, and Xiaohui Zhang.
\newblock Fast periodic graph exploration with constant memory.
\newblock {\em Journal of Computer and System Sciences}, 74(5):808--822, 2008.

\bibitem{gasieniec2007tree}
Leszek Gasieniec, Andrzej Pelc, Tomasz Radzik, and Xiaohui Zhang.
\newblock Tree exploration with logarithmic memory.
\newblock In {\em Proceedings of the eighteenth annual ACM-SIAM symposium on Discrete algorithms}, pages 585--594, 2007.

\bibitem{gkasieniec2008memory}
Leszek G{\k{a}}sieniec and Tomasz Radzik.
\newblock Memory efficient anonymous graph exploration.
\newblock In {\em Graph-Theoretic Concepts in Computer Science: 34th International Workshop, WG 2008, Durham, UK, June 30--July 2, 2008. Revised Papers 34}, pages 14--29. Springer, 2008.

\bibitem{icalphaeupler2019optimal}
Bernhard Haeupler, Fabian Kuhn, Anders Martinsson, Kalina Petrova, and Pascal Pfister.
\newblock Optimal strategies for patrolling fences.
\newblock In {\em 46th International Colloquium on Automata, Languages, and Programming (ICALP 2019)}. Schloss Dagstuhl-Leibniz-Zentrum fuer Informatik, 2019.

\bibitem{hsiang_algorithms_2004}
Tien-Ruey Hsiang, Esther~M. Arkin, Michael~A. Bender, Sándor~P. Fekete, and Joseph S.~B. Mitchell.
\newblock Algorithms for {Rapidly} {Dispersing} {Robot} {Swarms} in {Unknown} {Environments}.
\newblock {\em SpringerLink}, pages 77--93, 2004.

\bibitem{ilcinkas2008setting}
David Ilcinkas.
\newblock Setting port numbers for fast graph exploration.
\newblock {\em Theoretical Computer Science}, 401(1-3):236--242, 2008.

\bibitem{izumi2022deciding}
Taisuke Izumi, Kazuki Kakizawa, Yuya Kawabata, Naoki Kitamura, and Toshimitsu Masuzawa.
\newblock Deciding a graph property by a single mobile agent: One-bit memory suffices.
\newblock {\em arXiv preprint arXiv:2209.01906}, 2022.

\bibitem{kilibarda2017reduction}
Goran Kilibarda.
\newblock On reduction of automata in labyrinths.
\newblock {\em Publications de l'Institut Mathematique}, 101(115):47--63, 2017.

\bibitem{moore2011nature}
Cristopher Moore and Stephan Mertens.
\newblock {\em The nature of computation}.
\newblock OUP Oxford, 2011.

\bibitem{rabinovich2022emerging}
Dmitry Rabinovich and Alfred~M. Bruckstein.
\newblock Emerging cooperation on the road by myopic local interactions, 2022.

\bibitem{amir_rappel2023stigmergy}
Ori Rappel, Michael Amir, and Alfred~M Bruckstein.
\newblock Stigmergy-based, dual-layer coverage of unknown regions.
\newblock In {\em Proceedings of the 2023 International Conference on Autonomous Agents and Multiagent Systems}, pages 1439--1447, 2023.

\bibitem{shats2023competitive}
Alon Shats, Michael Amir, and Noa Agmon.
\newblock Competitive ant coverage: The value of pursuit.
\newblock In {\em 2023 IEEE/RSJ International Conference on Intelligent Robots and Systems (IROS)}, pages 6733--6740. IEEE, 2023.

\bibitem{shimoyama2022one}
Kohei Shimoyama, Yuichi Sudo, Hirotsugu Kakugawa, and Toshimitsu Masuzawa.
\newblock One bit agent memory is enough for snap-stabilizing perpetual exploration of cactus graphs with distinguishable cycles.
\newblock In {\em Stabilization, Safety, and Security of Distributed Systems: 24th International Symposium, SSS 2022, Clermont-Ferrand, France, November 15--17, 2022, Proceedings}, pages 19--34. Springer, 2022.

\end{thebibliography}
\bibliographystyle{plain}

\appendix
\newpage 
\newpage
\section{Appendix - Non-inductive Pseudocode}
\label{appendix:nonrecursivepseudocode}

Our $d$-dimensional grid graph patrolling algorithm,  \MakeMovekD, was inductively defined by the pseudocode given in \cref{alg:V=1 b=1-Patrolling Algorithm (k+1)D}. Here we explicitly write out the same algorithm in non-inductive fashion - see \cref{alg:V=1 b=1-Patrolling Algorithm 3+D}.  

\begin{algorithm}[!ht]
    \SetKw{KwGoTo}{goto}

    \SetKwFunction{step}{step}
    \SetKwFunction{movedownxk}{move\_down\_$x_{k+1}$}
    \SetKwFunction{moveupxk}{move\_up\_$x_{k+1}$}
    \SetKwFunction{MakeMove}{MakeMove}
    \SetKwFunction{MakeMoveTwoD}{MakeMove2D}

    \SetKw{And}{and}
    \SetKw{Not}{not}
    \SetKw{Or}{or}
    \SetKw{Continue}{continue}
    \SetKw{Break}{break}
    
    \SetArgSty{textrm}
    \BlankLine
    \SetKwProg{Fn}{Function}{}{} 
    \Fn{\MakeMove}{
    \DontPrintSemicolon
    \KwIn{$\{(\textit{l}_1, \textit{r}_1), \ldots (\textit{l}_d, \textit{r}_d)\}$, $mem \in \{0,1\}$.}
        \BlankLine
        $\step = (0, \ldots, 0)$ \tcp*{$d$ - dimensional zero vector}
        \BlankLine
        \uIf(\tcp*[f]{see bold up-arrows in \cref{fig:visualization.m=1.3d}}){$mem = 0\;\And\;l_1=0\;\And\;r_2=0$}{
        \label{code:starthandlespecialcorner1}
            \lIf{$r_3=0$}{
                $step[1] = 1$ \tcp*[f]{maxed $x_3$}
            \label{code:handlerow1table}
            }
            \lElse{
                $step[3] = 1$ \tcp*[f]{go up the $x_3$ axis}
            \label{code:handlerow2table}
            }
            $mem = 1$\;
            \Return \step, $mem$\;
        }
        
        \BlankLine
        \uIf{$mem = 1$} {
        $k = \underset{j > 2}\argmin\; l_j > 0$ \tcp*{find the first coordinate with $x_k > 1$}
        $\moveupxk=\;$\leIf{$r_1=l_2 = 0\;$}{1}{0}
        \BlankLine
        \uIf{$\moveupxk\;\And\;k<d\;\And\;r_k=0$}{
            \label{code:handlespecialcorner2}
            \lIf{$r_{k+1}=0$}{
                \label{code:handlespecialcorner6}
                $step[1] = -1$ \tcp*[f]{maxed $x_{k+1}$}
            } \uElse {
            \label{code:starthandlespecialcorner2}
                $step[k + 1] = 1$ \tcp*[r]{go up the $x_{k+1}$ axis}
                $mem = 0$
            }
            \Return \step, $mem$\;
            \label{code:endhandlespecialcorner2}
        }
        \BlankLine
            \uIf{$l_1 \neq 0\;\And\;r_2=0$}{
                \uIf{$l_3 \neq 0$}{
                \label{code:starthandlespecialcorner3}
                    $step[3] = -1$ \tcp*[r]{go down the $x_{3}$ axis}
                    \Return \step, $mem$
                }
                \label{code:endhandlespecialcorner3}
        \BlankLine
           \Return \MakeMoveTwoD($\{(\textit{l}_1, \textit{r}_1),  (\textit{l}_2, \textit{r}_2)\}, \textit{mem}$)
           \tcp*[r]{$x_3=1$}
            }
        \BlankLine
            $\movedownxk=\;$\leIf{$r_1 \neq 0\;\And\;l_2=0$}{1}{0}
        \BlankLine
            \uIf{$\movedownxk\;\And\;k<d\;\And\;r_k=0$}{
                \uIf{$l_{k+1} \neq 0$}{
                    $step[k+1] = -1$ \tcp*[r]{go down the $x_{k+1}$ axis}
                    \Return \step, $mem$\;
            \label{code:handlespecialcorner4}
                }
        \BlankLine
                \Return \MakeMoveTwoD($\{(\textit{l}_1, \textit{r}_1), (\textit{l}_2, \textit{r}_2)\}, \textit{mem}$)\tcp*[f]{$x_{k+1} = 1$}
            }
        }
        \BlankLine
        \Return \MakeMoveTwoD($\{(\textit{l}_1, \textit{r}_1),  (\textit{l}_2, \textit{r}_2)\}, \textit{mem}$)\tcp*{patrol axes $x_1$, $x_2$ using \cref{alg:V=1 b=1-Patrolling Algorithm 2D}}
    }

\caption{A $V=1$ sensing range, $d$-dimensional grid graph patrolling algorithm, using $1$ bit of internal memory.}
\label{alg:V=1 b=1-Patrolling Algorithm 3+D}	
\end{algorithm}

For simplicity, \cref{alg:V=1 b=1-Patrolling Algorithm 3+D} as given only works for $d$-dimensional grid graphs where $d \geq 2$; we do not handle the $d = 1$ case. Such a case can easily be handled by detecting the dimension of $\Dcube$ based on the sensory inputs $\{(\textit{l}_1, \textit{r}_1), \ldots (\textit{l}_d, \textit{r}_d)\}$ and calling \MakeMoveOneD (\cref{subfig:visualization.m=1.1d}) if $\Dcube$ is 1-dimensional. 


\section{Appendix - Patrolling Non-Grid Environments}
\label{appendix:patrollingingeneral}

We can consider a more general model where our robot, $\Robot$, is located in a finite connected subgraph of $\mathbb{Z}^d$,  $\mathcal{G} \subset \mathbb{Z}^d$, and can sense vertices in $\mathcal{G}$ at Manhattan distance $V$ from itself:

\begin{equation*}
\Sense{V}{\Position} \defeq \Set{\Position' - \Position}{\Position' = (x_1', x_2', \ldots x_d') \in \mathcal{G}, \lVert \Position' - \Position \rVert_1 \leq V}
\end{equation*}

The definition of patrolling can directly be extended to non-grid environments:

\begin{definition}
Let $\mathcal{G} \subset Z^d$ be a finite connected subgraph of $\mathbb{Z}^d$. An algorithm $\mathrm{ALG}$ is said to \textbf{patrol} $\mathcal{L} \subset \mathcal{G}$ if, given any initial position $v_1 \in \mathcal{L}$ and initial memory state, executing $\mathrm{ALG}$ causes $\Robot$ to visit all vertices  $v \in \mathcal{L}$ within a finite number of steps.
\end{definition}

Other core definitions, such as our definition of sensing regions (\cref{definition:sensingregions}), can similarly be carried over. \cref{fig:cartesian-product.examples} illustrates some algorithms for patrolling non-grid subgraphs of $\mathbb{Z}^d$.

\begin{figure}[ht]
    \centering
    \begin{subfigure}{0.48\textwidth}
        \centering
        \resizebox{\textwidth}{!}{
            \includegraphics[page=1]{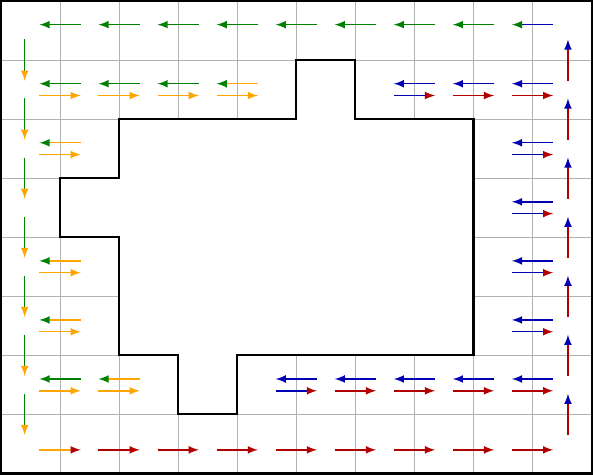}
        }
        \caption{}
        \label{subfig:cartesian-product.circle}
    \end{subfigure}
    \hfill
    \begin{subfigure}{0.48\textwidth}
        \centering
        \resizebox{\textwidth}{!}{
            \includegraphics[page=1]{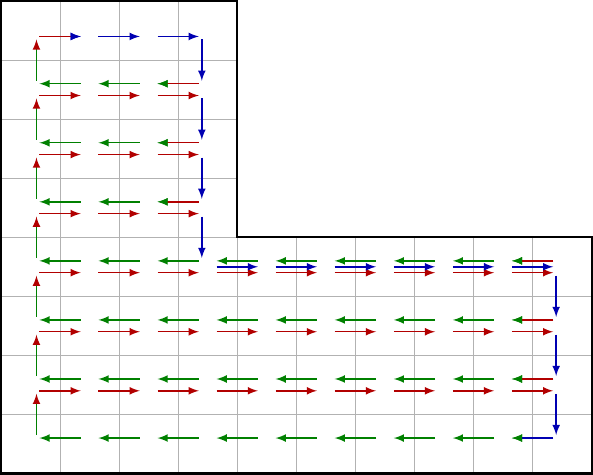}
        }
        \caption{}
        \label{subfig:cartesian-product.Lcorner}
    \end{subfigure}
    \caption{(a) illustrates a $2$-bit algorithm for patrolling a grid with a hole in the middle; (b) illustrates a $\log_2(3)$-bit algorithm (i.e., requiring $3$ unique states) for patrolling an L-shaped environment. Both algorithms assume sensing range $V = 1$. The illustrations should be interpreted like \cref{fig:visualization.m=1.1-2d}, but with additional colors representing new memory states. To reduce visual clutter, we omit ``transient'' arrows that are not part of the eventual steady-state patrol path the algorithms converge to.}
    \label{fig:cartesian-product.examples}
\end{figure}

It is natural to ask what the smallest amount of memory required to patrol $\mathcal{G}$ is with respect to some fixed sensing range, or what the smallest sensing range required to patrol $\mathcal{G}$ is with respect to some fixed memory size. We recount in \cref{theorem:trivialupperboundsonmemandvisrequiredtotraversegeneralsubgraphs}, (1) a folklore result that may be of interest, due to Fraigniaud et al. \cite{fraigniaud2005space,fraigniaud2005graph}. This result is not ours - but we reprove it here for completeness, and also because our different problem setting requires small changes to the argument. \cref{theorem:trivialupperboundsonmemandvisrequiredtotraversegeneralsubgraphs}, (2) is a simple corollary of our previous observations.

\begin{theorem}
Let $\mathcal{G}$ be a finite connected subgraph of $\mathbb{Z}^d$. Let $\mathrm{diam}(G)$ be the maximal distance between any pair of vertices $v,u \in \mathcal{G}$.

\begin{enumerate}
\item An algorithm exists that patrols $\mathcal{\mathcal{G}}$ using $\mathcal{O}(\mathrm{diam}(\mathcal{G}) \cdot \log d)$ bits of memory and $V = 1$ sensing range \cite{fraigniaud2005graph,fraigniaud2005space}.

\item An algorithm exists that patrols $\mathcal{G}$ using $0$ bits of memory and $V = \mathrm{diam}(\mathcal{G})$ sensing range if and only if $\mathcal{G}$ is Hamiltonian.
\end{enumerate}
\label{theorem:trivialupperboundsonmemandvisrequiredtotraversegeneralsubgraphs}
\end{theorem}

\begin{proof}
We first prove (1). Let $\Position$ be $\Robot$'s initial position. Note that any vertex in $\mathcal{G}$ can be reached by following some sequence of directions $(d_1, d_2, \ldots d_{\mathrm{diam}(\mathcal{G})})$  where $d_i$ denotes one of the $2d$ possible directions $\Robot$ can move in. Hence, we can patrol $\mathcal{G}$ by iterating over all possible sequences of the form $(d_1, d_2, \ldots d_{\mathrm{diam}(\mathcal{G})})$, at each iteration moving according to the sequence and backtracking to $\Position$ once we are done (note that it is possible that such a sequence sometimes requires us to move in an impossible direction - in which case we simply stay still). Since there are $\mathrm{diam}(\mathcal{G}) \cdot \log 2d$ possible direction sequences, the total memory required by this algorithm is $O(\mathrm{diam}(\mathcal{G}) \cdot \log 2d)$ bits:

\begin{enumerate}[label=(\alph*)]
    \item $\mathrm{diam}(\mathcal{G}) \cdot \log 2d$ bits to iterate over all direction sequences in lexicographic order,
    \item $\mathrm{diam}(\mathcal{G})$ bits to keep track of which moves were \textit{successful} in the current direction sequence, so that we can return to $\Position$ by reversing all successful moves, and
    \item $log(2\mathrm{diam}(\mathcal{G}))$ bits to keep track of how far along the direction sequence, or the backtracking process, we are. 
\end{enumerate}

We now prove (2). In one direction, the proof of \cref{lemma:coveringispatrollingwhenV=0} directly extends to non-grid graphs, showing that any algorithm that patrols $\mathcal{G}$ with $0$ bits of memory implies the existence of a Hamiltonian cycle. In the other direction, $V = \mathrm{diam}(\mathcal{G})$ enables $\Robot$ to always see the entirety of $\mathcal{G}$, thus this sensing range is sufficient to run an algorithm that patrols $\mathcal{G}$ by traversing some Hamiltonian cycle of it.
\end{proof}

To our knowledge, \cref{theorem:trivialupperboundsonmemandvisrequiredtotraversegeneralsubgraphs}, (1) is the best previously known bound also for patrolling $d$-dimensional grid graphs. This is perhaps surprising, since we show that just $1$ bit is necessary.


\end{document}